\documentclass[12pt]{article} %
\usepackage{times}
\usepackage{url}            %
\usepackage{booktabs}       %
\usepackage{amsfonts}       %
\usepackage{tikz}
\usetikzlibrary{arrows}
\usepackage{booktabs}
\usepackage{nicefrac}       %
\usepackage{microtype}      %
\usepackage{mkolar_definitions}
\usepackage{xspace}
\usepackage{multirow}
\usepackage{amsmath}
\usepackage{algorithm}
\usepackage{algorithmic}
\usepackage{color}
\usepackage{color, colortbl}
\usepackage{enumitem}
\usepackage{comment}
\usepackage{bm}
\usepackage{subcaption}
\usepackage[margin=1.0in]{geometry}
\usepackage{amsbsy}
\usepackage{wrapfig}
\usepackage{thmtools}

\usepackage{url}
\usepackage{authblk}
\usepackage{csquotes}
\usepackage{tikz}

\usepackage[colorlinks=true, linkcolor=blue, citecolor=blue]{hyperref}

\usepackage{natbib}

\usepackage{centernot}

\newcount\Comments  %
\definecolor{darkgreen}{rgb}{0,0.5,0}
\definecolor{darkred}{rgb}{0.7,0,0}
\definecolor{teal}{rgb}{0.3,0.8,0.8}
\definecolor{orange}{rgb}{1.0,0.5,0.0}
\definecolor{purple}{rgb}{0.8,0.0,0.8}
\newcommand{\kibitz}[2]{\ifnum\Comments=1{\textcolor{#1}{\textsf{\footnotesize #2}}}\fi}
\usetikzlibrary{positioning}

\definecolor{Gray}{gray}{0.9}

\newcommand{\KL}{\mathrm{KL}}

\newcommand{\pre}{\mathrm{pre}}

\definecolor{orange2}{rgb}{1.0, 0.27, 0.0}

\ifdefined\usebigfont

\usepackage{times}
\usepackage[fontsize=13pt]{scrextend}
\AtBeginDocument{\newgeometry{letterpaper,left=0.56in,right=0.56in,top=1.71in,bottom=1.77in}}

\else
\fi

\title{Understanding Reinforcement Learning-Based Fine-Tuning of Diffusion Models: A Tutorial and Review}

\author[1]{Masatoshi Uehara\thanks{\texttt{uehara.masatoshi@gene.com}}} 
\author[2]{Yulai Zhao\thanks{\texttt{yulaiz@princeton.edu}. Equal contribution.}}
\author[1]{Tommaso Biancalani}
\author[3]{Sergey Levine}
\affil[1]{Genentech}
\affil[2]{Princeton University} 
\affil[3]{University of California, Berkeley}

\begin{document}

\maketitle

\vspace{-5mm}
\begin{abstract}
{

This tutorial provides a comprehensive survey of methods for fine-tuning diffusion models to optimize downstream reward functions. While diffusion models are widely known to provide excellent generative modeling capability, practical applications in domains such as biology require generating samples that maximize some desired metric (e.g., translation efficiency in RNA, docking score in molecules, stability in protein). In these cases, the diffusion model can be optimized not only to generate realistic samples but also to maximize the measure of interest explicitly. Such methods are based on concepts from reinforcement learning (RL). We explain the application of various RL algorithms, including PPO, differentiable optimization, reward-weighted MLE, value-weighted sampling, and path consistency learning, tailored specifically for fine-tuning diffusion models. We aim to explore fundamental aspects such as the strengths and limitations of different RL-based fine-tuning algorithms across various scenarios, the benefits of RL-based fine-tuning compared to non-RL-based approaches, and the formal objectives of RL-based fine-tuning (target distributions). Additionally, we aim to examine their connections with related topics such as classifier guidance, Gflownets, flow-based diffusion models, path integral control theory, and sampling from unnormalized distributions such as MCMC. The code of this tutorial is available at \href{https://github.com/masa-ue/RLfinetuning_Diffusion_Bioseq}{https://github.com/masa-ue/RLfinetuning\_Diffusion\_Bioseq.} } 
\end{abstract}

\vspace{-5mm}
\section*{Introduction}

Diffusion models \citep{sohl2015deep,ho2020denoising,song2020denoising} are widely recognized as powerful tools for generative modeling. They are able to accurately model complex distributions by closely emulating the characteristics of the training data. There are many applications of diffusion models in various fields, including computer vision \citep{podell2023sdxl}, natural language processing \citep{austin2021structured}, biology \citep{avdeyev2023dirichlet,stark2024dirichlet,li2023latent}, chemistry \citep{jo2022score,xu2022geodiff,hoogeboom2022equivariant}, and biology \citep{avdeyev2023dirichlet,stark2024dirichlet,campbell2024generative}.  

While diffusion models exhibit significant power in capturing the training data distribution, there's often a need to customize these models for particular downstream reward functions. For instance, in computer vision, Stable Diffusion~\citep{Rombach_2022_CVPR} serves as a strong backbone pre-trained model. However, we may want to fine-tune it further by optimizing downstream reward functions such as aesthetic scores or human-alignment scores \citep{black2023training,fan2023dpok}. Similarly, in fields such as biology and chemistry, various sophisticated diffusion models have been developed for DNA, RNA, protein sequences, and molecules, effectively modeling biological and chemical spaces. Nonetheless, biologists and chemists typically aim to optimize specific downstream objectives such as cell-specific expression in DNA sequences \citep{gosai2023machine,lal2024reglm,sarkar2024designing}, translational efficiency/stability of RNA sequences \citep{castillo2021machine,agarwal2022genetic}, stability/bioactivity of protein sequence \citep{frey2023protein,widatalla2024aligning} or QED/SA scores of molecules \citep{zhou2019optimization}. 

\begin{figure}[!t]
    \centering
    \includegraphics[width=0.9\linewidth]{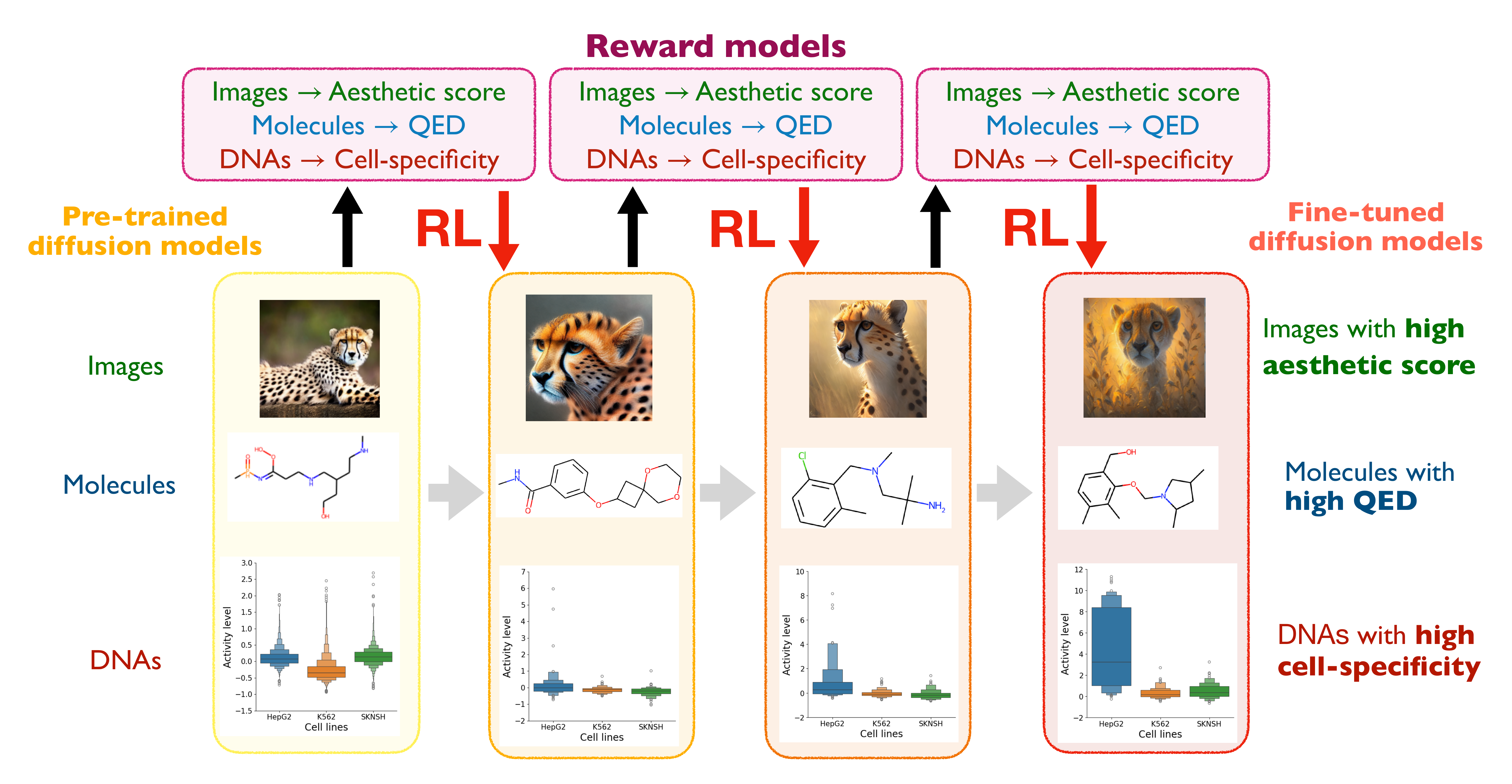}
    \caption{{ Illustrative examples of RL-based fine-tuning, aimed at optimizing pre-trained diffusion models to maximize downstream reward functions.}  }
    \label{fig:enter-label}
\end{figure}

To achieve this goal, numerous algorithms have been proposed for fine-tuning diffusion models via reinforcement learning (RL) (e.g., \citet{black2023training,fan2023dpok,clark2023directly,prabhudesai2023aligning,uehara2024finetuning}), aiming to optimize downstream reward functions. { RL is a machine learning paradigm where agents learn to make sequential decisions to maximize reward signals \citep{sutton2018reinforcement,agarwal2019reinforcement}.
 In our context, RL naturally emerges as a suitable approach due to the sequential structure inherent in diffusion models, where each time step involves a ``decision'' corresponding to how the sample is denoised at that step. This tutorial aims to review recent works for readers interested in understanding the fundamentals of RL-based fine-tuning from a holistic perspective, including the advantages of RL-based fine-tuning over non-RL approaches, the pros and cons of different RL-based fine-tuning algorithms, the formalized goal of RL-based fine-tuning, and its connections with related topics such as classifier guidance. }

The content of this tutorial is primarily divided into three parts. In addition, as an implementation example, we also release the code that employs RL-based fine-tuning for guided biological sequences (DNA/RNA) generation at \href{https://github.com/masa-ue/RLfinetuning_Diffusion_Bioseq}{https://github.com/masa-ue/RLfinetuning\_Diffusion\_Bioseq}.  
\begin{enumerate}
    \item We aim to provide a comprehensive overview of current algorithms. Notably, given the sequential nature of diffusion models, we can naturally frame fine-tuning as a reinforcement learning (RL) problem within Markov Decision Processes (MDPs), as detailed in \pref{sec:entropy} and \ref{sec:theory_RL}. Therefore, we can employ any off-the-shelf RL algorithms such as PPO \citep{schulman2017proximal}, differentiable optimization (direct reward backpropagation), weighted MLE \citep{peters2010relative,peng2019advantage}, value-weighted sampling (close to classifier guidance in \citet{dhariwal2021diffusion}), and path consistency learning \citep{nachum2017bridging}. We discuss these algorithms in detail in \pref{sec:sum_planning} and  \ref{sec:sum_planning_conservative}. {Instead of merely outlining each algorithm, we aim to present both their advantages and disadvantages so readers can select the most suitable algorithms for their specific purposes.} 
\vspace{-2mm}
\item  We categorize various fine-tuning scenarios based on how reward feedback is acquired in \pref{sec:settings}. This distinction is pivotal for practical algorithm design. 
For example, if we can access accurate reward functions,  computational efficiency would become our primary focus. However, in cases where reward functions are unknown, it is essential to learn them from data with reward feedback, 
leading us to take feedback efficiency and distributional shift into consideration as well. Specifically, when reward functions need to be learned from static offline data without any online interactions, we must address the issue of overoptimization, where fine-tuned models are misled by out-of-distribution samples, and generate samples with low genuine rewards. 
This is crucial because, in an offline scenario, the coverage of offline data distribution with feedback is limited; hence, the out-of-distribution region could be extensive \citep{uehara2024bridging}. 
\vspace{-2mm}
\item We provide a detailed discussion on the relationship between RL-based fine-tuning methods and closely related methods in the literature, such as classifier guidance \citep{dhariwal2021diffusion} in \pref{sec:connection}, flow-based diffusion models \citep{liu2022flow,lipman2022flow,tong2023conditional} in \pref{sec:flow_matching}, sampling from unnormalized distributions \citep{zhang2021path} in \pref{sec:sampling}, Gflownets \citep{bengio2023gflownet} in \pref{sec:gflownet}, and path integral control theory \citep{theodorou2010generalized,williams2017model,kazim2024recent} in \pref{sec:path_integral}. We summarize the key messages as follows. 
{ 
\begin{itemize}
\item \pref{sec:gflownet}: The losses used in Gflownets are fundamentally equivalent to those derived from a specific RL algorithm called path consistency learning.
    \item \pref{sec:connection}: Classifier guidance employed in conditional generation is regarded as a specific RL-based fine-tuning method, which we call value-weighted sampling. As formalized in \citet{zhao2024adding}, this observation indicates that any off-the-shelf RL-based fine-tuning algorithms (e.g., PPO and differentiable optimization) can be applied to conditional generation.
    \item \pref{sec:sampling}: Sampling from unnormalized distributions, often referred to as Gibbs distributions, is an important and challenging problem in diverse domains. While MCMC methods are traditionally used for this task, recognizing its similarity to the objectives of RL-based fine-tuning suggests that off-the-shelf RL algorithms can also effectively address the challenge of sampling from unnormalized distributions.
\end{itemize}
} 
\end{enumerate}

\tableofcontents

\section{Preliminaries}

In this section, we outline the fundamentals of diffusion models and elucidate the objective of fine-tuning them.

\subsection{Diffusion Models}

We present an overview of denoising diffusion probabilistic models (DDPM) \citep{ho2020denoising}. For more details, refer to \citet{yang2023diffusion,cao2024survey,chen2024overview,tang2024score}. 

In diffusion models, the objective is to develop a deep generative model that accurately captures the true data distribution. Specifically, denoting the data distribution by $p_{\pre} \in \Delta(\Xcal)$ where $\Xcal$ is an input space, a DDPM aims to approximate $p_{\pre}$ using a parametric model structured as  $$p(x_0;\theta)=\int p(x_{0:T};\theta)d x_{1:T}, \mathrm{where}\, p(x_{0:T};\theta) =p_{T+1}(x_T;\theta)\prod_{t=T}^1 p_{t}(x_{t-1}|x_t;\theta).$$ When $\Xcal$ is an Euclidean space (in $\RR^d$), the forward process is modeled as the following dynamics: 
\begin{align*}
   p_{T+1}(x_T) = \Ncal(0, I),\, p_{t}(x_{t-1}|x_t;\theta) = \Ncal(\rho(x_t,t;\theta), \sigma^2(t)\times I), 
\end{align*} 
where $\Ncal(\cdot,\cdot)$ denotes a normal distribution, $I$ is an identity matrix and $\rho:\RR^d \times [0,T]\to \RR^d$.  In DDPMs, we aim to obtain a set of policies (i.e., denoising process) $\{p_t\}_{t=T+1}^1$, $p_t: \Xcal \to \Delta(\Xcal)$ such that $p(x_0;\theta)\approx p_{\pre}(x_0)$. Indeed, by optimizing the variational bound on the negative log-likelihood, we can derive such a set of policies. For more details, refer to \pref{sec:training}. 

Hereafter, we consider a situation where we have a pre-trained diffusion model that is already trained on a large dataset, such that the model can accurately capture the underlying data distribution. We refer to the pre-trained policies as $\{p^{\mathrm{pre}}_t(\cdot|\cdot)\}_{t=T+1}^1$, and to the marginal distribution at $t=0$ induced by the pre-trained diffusion model as $p_{\pre}$. In other words, 
\begin{align*}
    p^{\pre} (x_{0}) = \int p^{\pre}_{T+1}(x_T)\prod_{t=T}^1 p^{\pre}_{t}(x_{t-1}|x_t) d x_{1:T}. 
\end{align*}

\begin{remark}[Non-Euclidean space]
For simplicity, we typically assume that the domain space is Euclidean. However, we can easily extend most of the discussion to a more general space, such as a Riemannian manifold \citep{de2022riemannian} or discrete space \citep{austin2021structured,campbell2022continuous,benton2024denoising,lou2023discrete}.
\end{remark}

\begin{remark}[Conditional generative models]
Pre-trained models can be conditional diffusion models, such as text-to-image diffusion models \citep{ramesh2022hierarchical}. The extension is straightforward: augmenting the input spaces of policies with an additional space on which we want to condition. More specifically, by denoting that space by $c \in \Ccal$, each policy becomes $p_t(x_t|x_t,c;\theta):\Xcal \times \Ccal \to \Delta(\Xcal)$. 
\end{remark}

\begin{remark}[Extension to Continuous-time diffusion models]
In this tutorial, our discussion on fine-tuning diffusion models will be primarily  formulated on the discrete-time formulation, as we did above Nonetheless, much of our discussion is also applicable to continuous-time diffusion models, as formalized in \citet{uehara2024finetuning}
\end{remark}

\subsubsection{Score-Based Diffusion Models (Optional) } \label{sec:training}

We briefly discuss how to train diffusion models in a continuous-time framework \citep{song2021maximum}. In our tutorial, this section is mainly used to discuss several algorithms (e.g., value-weighted sampling in \pref{sec:value}) and their relationship with flow-based diffusion models  (\pref{sec:flow_matching}) later. Therefore, readers may skip it based on their individual needs.

{ The training process of diffusion models can be summarized as follows. Our objective is to train a sequential mapping from a known noise distribution to a data distribution, formalized through a stochastic differential equation (SDE). Firstly, we define a forward (fixed) SDE that maps from the data distribution to the noise distribution. Then, a time-reversal SDE is expressed as an SDE, which includes the (unknown) score function. Now, by learning this unknown score function from the training data, the time-reversal SDE can be utilized as a generative model. Here are the details.} 

\paragraph{Forward and time-reversal SDE.}
Firstly, we introduce a forward (a.k.a., reference) Stochastic differential equations (SDE) from $0$ to $T$. A common choice is a variance-preserving (VP) process:
\begin{align}\label{eq:reference}
  t \in [0,T];  dx_t = -0.5 x_t dt +  dw_t, x_0 \sim p^{\pre}(x), 
\end{align} 
where $dw_t$ represents standard Brownian motion. Here are two crucial observations:
\begin{itemize}
    \item As $T$ approaches $\infty$, the limiting distribution is $\Ncal(0,\mathbb{I})$. 
    \item The time-reversal SDE \citep{anderson1982reverse}, which preserves the marginal distribution, is expressed as follows: 
    \begin{align}\label{eq:time_reversal}
    t \in [0,T];    dz_t = [0.5z_t + \nabla \log q_{T-t}(z_t)]dt + dw_t. 
    \end{align}
    Here, $q_t\in \Delta(\RR^d)$ denotes the marginal distribution at time $t$ induced by the reference SDE \eqref{eq:reference}. Notably, the marginal distribution of $z_{T-t}$ is the same as the one of $x_t$ induced by the reference SDE. 
\end{itemize}
These observations suggest that with sufficiently large $T$, starting from $\Ncal(0,\mathbb{I})$ and following the time-reversal SDE \eqref{eq:time_reversal}
, we are able to sample from the data distribution (i.e., $p^\pre$) at the terminal time point $T$.

\paragraph{Training score functions from the training data.}
Now, the remaining question is how to estimate the marginal score function $\nabla \log q_t(\cdot)$ in \eqref{eq:time_reversal}. This can be computed using the following principles:
\begin{itemize}
    \item We have  $
     \nabla \log q_t(x_t)= \EE_{x_0 \sim p^{\pre} }[\nabla_{x_t} \log q_{t|0}(x_t\mid x_0)],$  where $q_{t|0}$ represents the conditional distribution of $x_t$ given $x_0$ induced by the reference SDE. 
\item  The conditional score function $\nabla_{x_t} \log q_{t|0}(x_t \mid x_0)$ can be analytically derived as
\begin{align}\label{eq:conditional}
    \frac{\mu^{\diamond}_t x_0 - x_t }{ \{ \sigma^{\diamond}_t \}^2},\,\mu^{\diamond}_t = \exp(-0.5t)x_0,\, \{\sigma^{\diamond}_t \}^2 =1- \exp(-0.5 t). 
\end{align}
\end{itemize} 
Subsequently, by defining a parameterized model $s_{\theta}:\RR^d \times [0, T]\to \RR^d$ and utilizing the following weighted regression, the marginal score is estimated:
\begin{align}\label{eq:loss_diffusion}
  \hat \theta_{\pre} =  \argmin_{\theta} \EE_{t\sim \mathrm{Uni}([0,T]),x_t \sim q_{t\mid 0}(x_t|x_0), x_0 \sim p^{\pre} }[ \lambda(t)  \| \nabla_{x_t} \log q_{t\mid 0}(x_t \mid x_0) - s_{\theta}(x_t,t) \| ^2     ]
\end{align}
where $\lambda:[0,T]\to \RR$ is a weighted function. 

Note we typically use another parametrization well. From \eqref{eq:conditional}, we can easily see that this score is estimated by introducing networks $ \epsilon_{\theta}:\RR^d \times [0,T]\to \RR^d$, which aims to estimate the noise $\frac{x_t - \mu^{\diamond}_t x_0  }{ \sigma^{\diamond}_t}$: 
\begin{align}\label{eq:loss_score}
 \bar \theta_{\pre} =  \argmin_{\theta} \EE_{t\sim \mathrm{Uni}([0,T]),x_t \sim \mu^{\diamond}_t x_0 + \epsilon \sigma^{\diamond}_t,\epsilon \sim \Ncal(0,\mathbb{I}), x_0 \sim p^{\pre} }[ \lambda(t)/\{ \sigma^{\diamond}_t\}^2 \|\epsilon - \epsilon_{\theta}(x_t,t) \| ^2      ]. 
\end{align}
 Then, by denoting the training data as $\{x^{(i)}_0\}$ and setting $\lambda(t) =\{\sigma^{\diamond}_t\}^2 $, the actual loss function from the training data is 
\begin{align}\label{eq:epsilon}
   \argmin_{\theta} \sum_{i=1} \EE_{t\sim \mathrm{Uni}([0,T]),x^{(i)}_t \sim \mu^{\diamond}_t x^{(i)}_0 + \epsilon \sigma^{\diamond}_t, \epsilon \sim \Ncal(0,\mathbb{I})}[  \|\epsilon - \epsilon_{\theta}(x^{(i)}_t,t) \| ^2      ]. 
\end{align}

\paragraph{Inference time.} Once this $\hat \theta_{\pre}$ (or $ \bar \theta_{\pre}$) is learned, we insert it into the time-reversal SDE \eqref{eq:time_reversal} and use it as a generative model. Ultimately, with standard discretization, we obtain:
\begin{align*}
    \rho(x_t,t;\hat \theta_{\pre}) &= x_t + [ 0.5 x_t + s_{\hat \theta_{\pre}}(x_{t},T-t) ] (\delta t)  ,\quad  \{\sigma^{\diamond}_t(t)\}^2 = (\delta t). 
\end{align*} 
Equivalently, this is
\begin{align}
   \rho(x_t,t;\hat \theta_{\pre})  =  x_t + \left [ 0.5 x_t  -  1/\sigma^{\diamond}_t \times \epsilon_{ \bar \theta_{\pre}}(x_t,T-t)    \right] (\delta t) , \label{eq:vp}
\end{align}
where $(\delta t)$ denotes the discretization step.

\paragraph{Equivalence to DDPM.}

{ The objective function derived here is equivalent to the one formulated based on variational inference in the discretized formulation, which is commonly referred to as DDPMs 
 \citep{ho2020denoising}.} In DDPMs \citep{ho2020denoising}, we often see the following form:  
\begin{align*}
    \rho(x_t,t; \theta)=  \frac{1}{\sqrt{\alpha_t}}x_t - \frac{1-\alpha_t }{\sqrt{1- \bar \alpha_t }\sqrt{\alpha_t}}\epsilon_{ \theta}(x_t,T-t),\quad \bar \alpha_t = \prod_{k=1}^t \alpha_k. 
\end{align*}
When $\alpha_t =1-(\delta t)$, the above is equivalent to \eqref{eq:vp} when $\delta t$ goes to $0$ by noting $1/\sqrt{\alpha_t}\approx 1 + 0.5(\delta t)$.

\subsection{Fine-Tuning Diffusion Models with RL} 

Importantly, our focus on RL-based fine-tuning distinguishes itself from the standard fine-tuning methods. Standard fine-tuning typically involves scenarios where we have pre-trained models (e.g., diffusion models) and new training data $\{x^{(i)},y^{(i)}\}$.  In such cases, the common approach for fine-tuning is to retrain diffusion models with the new training data using the same loss function employed during pre-training. In sharp contrast, RL-based fine-tuning directly employs the downstream reward functions as the primary optimization objectives, making the loss functions different from those used in pre-training. 

{ Hereafter, we start with a concise overview of RL-based fine-tuning. Then, before delving into specifics, we discuss simpler non-RL alternatives to provide motivation for adopting RL-based fine-tuning. }

\subsubsection{Brief Overview: Fine-tuning with RL}\label{sec:formalized_goal}

In this article, we explore the fine-tuning of pre-trained diffusion models to optimize downstream reward functions $r:\mathbb{R}^d \rightarrow \mathbb{R}$. In domains such as images, these backbone diffusion models to be fine-tuned include Stable Diffusion~\citep{Rombach_2022_CVPR}, while the reward functions are aesthetic scores and alignment scores \citep{clark2023directly,black2023training,fan2023dpok}. More examples are detailed in the introduction. 
These rewards are often unknown, necessitating learning from data with feedback: $\{x^{(i)},r(x^{(i)})\}$. We will explore this aspect further in Section~\ref{sec:settings}. Until then, we assume $r$ is known.

{ Now, readers may wonder about the objectives we aim to achieve during the fine-tuning process. A natural approach is to define the optimization problem:  
\begin{align}\label{eq:RL_goal}
  \argmax_{q\in \Delta(\Xcal)}  \EE_{x \sim q}[r(x)]
\end{align}
where $q$ is initialized with a pre-trained diffusion model $p^{\pre} \in \Delta(\Xcal)$. In this tutorial, we will detail the procedure of solving \eqref{eq:RL_goal} with RL in the upcoming sections. In essence, we leverage the fact that diffusion models are formulated as a sequential decision-making problem, where each decision corresponds to how samples are denoised. 
} 

{ Although the above objective function \pref{eq:RL_goal} is reasonable, the resulting distribution might deviate too much from the pre-trained diffusion model. } To circumvent this issue, a natural way is to add penalization against pre-trained diffusion models. Then, the target distribution is defined as:
\begin{align}\label{eq:original_target}
\argmax_{q\in \Delta(\Xcal)} \EE_{x \sim q}[r(x)] - \alpha \KL(q\| p^{\pre}).
\end{align}
Notably, \eqref{eq:original_target} reduces to the following distribution:
\begin{align}\label{eq:original_target2}
p_{r}(\cdot):=\frac{\exp(r(\cdot)/\alpha )p^{\pre}(\cdot)}{\int \exp(r(x)/\alpha)p^{\pre}(x)dx}.
\end{align}
Here, the first term in \eqref{eq:original_target} corresponds to the mean reward, which we want to optimize in the fine-tuning process. The second term in \eqref{eq:original_target2} serves as a penalty term, indicating the deviation of $q$ from the pre-trained model. The parameter $\alpha$ controls the strength of this regularization term. The proper choice of $\alpha$ depends on the task we are interested in.

\subsubsection{Motivation for Using RL over Non-RL Alternatives}

 To achieve our goal of maximizing downstream reward functions with diffusion models, readers may question whether alternative approaches can be employed apart from RL. Here, we investigate these potential alternatives and explain why RL approaches may offer advantages over them. 

\paragraph{Rejection sampling.}  One approach involves generating multiple samples from pre-trained diffusion models and selecting only those with high rewards. This method, called rejection sampling, operates without needing fine-tuning. However, rejection sampling is effective primarily when the pre-trained model already has a high probability of producing high-reward samples. It resembles sampling from a prior distribution to obtain posterior samples (in this case, high-reward points). This approach works efficiently when the posterior closely matches the prior but can become highly inefficient otherwise.  In contrast, by explicitly updating weight in diffusion models, RL-based fine-tuning allows us to obtain these high-reward samples, which are seldom generated by pre-trained models. 

{
\paragraph{Conditional diffusion models (classifier-free guidance).} In conditional generative models, the general goal is to sample from $p(x|c)$, where $x$ is the output and $c$ denotes the conditioning variable. For example, in text-to-language diffusion models, $c$ is a text, and $x$ is the generated image. Similarly, in the context of protein engineering for addressing inverse folding problems, models often define $c$ as the protein backbone structure and $x$ as the corresponding amino acid sequence. Here, using the training data $\{c^{(i)},  x^{(i)}\}$, the model is trained by using the loss function: 
\begin{align*}
   \argmin_{\theta} \sum_{i} \EE_{t\sim \mathrm{Uni}([0,T]),x^{(i)}_t \sim \mu^{\diamond}_t x^{(i)}_0 + \epsilon \sigma^{\diamond}_t, , \epsilon \sim \Ncal(0,\mathbb{I})}[  \|\epsilon - \epsilon_{\theta}(x^{(i)}_t,c^{(i)}, t) \| ^2      ], 
\end{align*}
where the denoising function $\epsilon_{\theta}$ additionally receives the conditioning information $c^{(i)}$ as input. In practice, a variety of improvements such as classifier-free guidance \citep{ho2022classifier} can further improve the model's ability to learn the conditional distribution $p(x|c)$.

These conditional generative models can be used to optimize down-stream rewards by conditioning on the reward values, then sampling $x$ conditioned on high reward values \citep{krishnamoorthy2023diffusion,yuan2023reward}. While this method is, in principle, capable of generating plausible $x$ values across a range of reward levels within the training data distribution, it is not the most effective optimization strategy. This is primarily because high-reward inputs frequently reside in the tails of the training distribution or even beyond it. Consequently, this method may not effectively generate high-reward samples that lie outside the training data distribution. In contrast, RL-based fine-tuning has the capability to generate samples with higher rewards beyond the training data. This is achieved by explicitly maximizing reward models learned from the training data and leveraging their extrapolative capabilities of reward models, as theoretically formalized and empirically observed in \citet{uehara2024bridging}.  } 

\paragraph{Reward-weighted training.} Another alternative approach is to use a reward-weighted version of the standard training loss for diffusion models. Suppose that we have data $\{x^{(i)},r(x^{(i)})\}$. Then, after learning a reward $\hat r:\Xcal \to \RR$ with regression from the data, to achieve our goal, it looks natural to use a reward-weighted version of the training loss for diffusion models \eqref{eq:loss_score}, i.e., 
\begin{align*}
    \argmin_{\theta} \sum_{i} \EE_{t\sim \mathrm{Uni}([0,T]),x^{(i)}_t \sim \mu^{\diamond}_t x^{(i)}_0 + \epsilon \sigma^{\diamond}_t, \epsilon \sim \Ncal(0,\mathbb{I})}[\hat r(x^{(i)}_0)  \|\epsilon - \epsilon_{\theta}(x^{(i)}_t,t) \| ^2      ]. 
\end{align*}

There are two potential drawbacks to this approach. First, in practice, it may struggle to generate samples with higher rewards beyond the training data. As we will explain later, many RL algorithms are more directly focused on optimizing reward functions, which are expected to excel in obtaining samples with high rewards not observed in the original data, { as empirically observed in \citet{black2023training}.} Second, when fine-tuning a conditional diffusion model $p(x|c)$, the alternative approach here requires a pair of $\{c^{(i)},x^{(i)}\}$ during fine-tuning to ensure the validity of the loss function. When we only have data $\{x^{(i)},r(x^{(i)})\}$ but not $\{c^{(i)},x^{(i)},r(x^{(i)})\}$, this implies that we might need to solve an inverse problem from $x$ to $c$, which can often be challenging. In contrast, in these scenarios, RL algorithms, which we will introduce later, can operate without needing such pairs $\{c^{(i)},x^{(i)}\}$, as long as we have learned reward functions $\hat r$.   

{ Finally, it should be noted that reward-weighted training technically falls under the broader category of RL methods. It shares a close connection with ``reward-weighted MLE'' introduced in  \pref{sec:reward-weighted}, as discussed later. Employing this reward-weighted MLE helps address the second concern of ``reward-weighted training'' mentioned earlier. }

\section{Brief Overview of Entropy-Regularized MDPs} \label{sec:overview}

{ In this tutorial, we explain how fine-tuning diffusion models can be naturally formulated as an RL problem in entropy-regularized MDPs. This perspective is natural because RL involves sequential decision-making, and a diffusion model is formulated as a sequential problem where each denoising step is a decision-making process. To connect diffusion models with RL, we begin with a concise overview of RL in standard entropy-regularized MDPs \citep{haarnoja2017reinforcement,neu2017unified,geist2019theory,schulman2017equivalence}. }

\subsection{MDPs}

An MDP is defined as follows: $\{ \Scal, \Acal, \{P^{\mathrm{tra}}_t\}_{t=0}^T, \{r_t\}_{t=0}^T,p_0  \} $ where $\Scal$ is the state space, $\Acal$ is the action space, $P^{\mathrm{tra}}_t$ is a transition dynamic mapping: $\Scal \times \Acal \to \Delta(\Scal)$, $r_t:\Scal \times \Acal \to \RR$ denotes reward received at $t$ and $p_0$ is an initial distribution over $\Scal$. A policy $\pi_t:\Scal \to \Delta(\Acal)$ is a map from any state $s \in \Scal$ to the distribution over actions. The standard goal in RL is to solve  
\begin{align}\label{eq:standard_RL}
   \argmax_{\{\pi_t\}} \EE_{\{\pi_t\} }\left [\sum_{t=0}^T r_t(s_t,a_t) \right]
\end{align}
where $\EE_{\{\pi_t\}}[\cdot]$ is the expectation induced both policy $\pi$ and the transition dynamics as follows: $s_0 \sim p_0, a_0 \sim \pi_0(\cdot|s_0), s_1 \sim P^{\mathrm{tra}}_0(\cdot|s_0,a_0),\cdots$.  { As we will soon detail in the next section (\pref{sec:entropy}), diffusion models can naturally be framed as MDPs as each policy corresponds to a denoising process in diffusion models. }    

In entropy-regularized MDPs, we consider the following regularized objective instead:
\begin{align}\label{eq:original_goal}
 \{ \pi^{\star}_t\} = \argmax_{\{\pi_t\}} \EE_{\{\pi_t\}}\left [\sum_{t=0}^T r_t(s_t,a_t) - \alpha \KL(\pi_t(\cdot|s_t),\pi'_t(\cdot|s_t) ) \right ]
\end{align}
where $\pi': \Scal \to \Delta(\Acal)$ is a certain reference policy. The arg max solution is often called a set of soft optimal policies. { Compared to a standard objective \eqref{eq:standard_RL}, here we add KL terms against reference policies. This addition aims to ensure that soft optimal policies closely align with the reference policies. In the context of fine-tuning diffusion models, these reference policies correspond to the pre-trained diffusion models, as we aim to maintain similarity between the fine-tuned and pre-trained models.  }

This entropy-regularized objective in \eqref{eq:original_goal} has been widely employed in RL literature due to several benefits \citep{levine2018reinforcement}. For instance, in online RL, it is known that these policies have good exploration properties by setting reference policies as uniform policies \citep{fox2015taming,haarnoja2017reinforcement}. In offline RL, \citet{wu2019behavior} suggests using these policies as conservative policies by setting reference policies close to behavior policies (policies used to collect offline data). Additionally, in inverse RL, this soft optimal policy is used as an expert policy in scenarios where rewards are unobservable, only trajectories from expert policies are available (typically referred to as maximum entropy RL as \citet{ziebart2008maximum,wulfmeier2015maximum,finn2016guided}).

{ \subsection{Key Concepts: Soft Q-functions, Soft  Bellman Equations. }  

The crucial question in RL is how to devise algorithms that effectively solve the optimization problem \eqref{eq:original_goal}. These algorithms are later used as fine-tuning algorithms of diffusion models. To see these algorithms, we rely on several critical concepts in entropy-regularized MDPs. Specifically, soft-optimal policies (i.e., solutions to \eqref{eq:original_goal}) can be expressed analytically as a blend of soft Q-functions and reference policies. Furthermore, these soft Q-functions are defined as solutions to equations known as soft Bellman equations. We elaborate on these foundational concepts below. 
}

\paragraph{Soft Q-functions and soft optimal policies.} { Soft optimal policies are expressed as a blend of soft Q-functions and reference policies.} To see it, we define the soft Q-function as follows:
\begin{align}\label{eq:soft-q-functions}
   q_t(s_t,a_t) = \EE_{\{\pi^{\star}_t\} }\left [\sum_{k=t}^T r_k(s_k,a_k) - \alpha   \KL(\pi^{\star}_{k+1}(\cdot|s_{k+1} )\|\pi'_{k+1} (\cdot|s_{k+1}))   |s_t,a_t \right ]. 
\end{align}
Then, by comparing \eqref{eq:soft-q-functions} and \eqref{eq:original_goal}, we clearly have 
\begin{align}\label{eq:soft-optimal-q}
   \pi^{\star}_t =  \argmax_{\pi \in [\Xcal \to \Delta(\Xcal) ] } \EE_{a_t \sim \pi(s_t)} [q_t(s_t,a_t) - \alpha \KL(\pi(\cdot |s_t) \| \pi'_t(\cdot\mid s_t) |s_t ]. 
\end{align}
Hence, by calculating the above explicitly, a soft optimal policy in \eqref{eq:original_goal} is described as follows:
\begin{align}\label{eq:sample}
    \pi^{\star}_t(\cdot |s) \propto \frac{\exp(q_t(s, \cdot )/\alpha) \pi'_t(\cdot |s) }{ \int \exp(q_t(s,a)/\alpha)\pi'_t(a |s)\mathrm{d}a } 
\end{align}

\paragraph{Soft Bellman equations.} { We have already defined soft Q-functions in \eqref{eq:soft-q-functions}. However, this form includes the soft optimal policies.} Actually, without using soft optimal policies, the soft Q-function satisfies the following recursive equation (a.k.a. soft Bellman equation):
\begin{align}\label{eq:soft_bellman}
    q_t(s_t,a_t) = \EE_{ \{\pi^{\star}_t\} }\left [r(s_t,a_t) + \alpha  \log \left \{  \int \exp(q_{t+1} (s_{t+1} ,a)/\alpha) \pi'_t(a |s_{t+1 })\mathrm{d}a \right \} \mid s_t, a_t \right ]. 
\end{align}
This is proven by noting we recursively have 
\begin{align*}
    q_t(s_t,a_t)=\EE_{ \{ \pi^{\star}_t\} }[r_t (s_t,a_t) + q_{t+1}(s_{t+1},a_{t+1}) - \alpha \KL(\pi^{\star}_{t+1}(\cdot|s_{t+1}),\pi'_{t+1}(\cdot|s_{t+1} ))   |s_t,a_t ] 
\end{align*}
By substituting \eqref{eq:sample} into the above, we obtain the soft Bellman equation \eqref{eq:soft_bellman}. 

\paragraph{Soft value functions.} { So far, we have defined the soft Q-functions, which depend on both states and actions. We can now introduce a related concept that depends solely on states, termed the soft value function. } The soft value function is defined as follows:  
\begin{align*}
   v_t(s_t) = \EE_{\{\pi^{\star}_t\} }\left [\sum_{k=t}^T r_k(s_k,a_k) - \alpha   \KL(\pi^{\star}_{k}(\cdot|s_{k} )\|\pi'_{k} (\cdot|s_{k}))   |s_t\right ]. 
\end{align*}
Then, the soft optimal policy in \eqref{eq:soft-optimal-q} is also written as 
\begin{align}\label{eq:sample2}
    \pi^{\star}_t(\cdot |s) \propto \frac{\exp(q_t(s, \cdot )/\alpha) \pi'_t(\cdot |s) }{\exp(v_t(s)/\alpha )  }  
\end{align}
because we have 
\begin{align*}
    \exp\left(\frac{v_t(s)}{\alpha} \right) = \int  \exp\left(\frac{q_t(s,a)}{\alpha} \right)\pi'_t(a\mid s)\mathrm{d}a.  
\end{align*}
Then, substituting the above in the soft Bellman equation \eqref{eq:soft_bellman}, it is written as 
\begin{align*}
    q_t(s_t,a_t)= \EE_{ \{\pi^{\star}_t\} }[r(s_t,a_t )+   v_{t+1}(s_{t+1})  |s_t,a_t ]. 
\end{align*}

\paragraph{Algorithms in entropy-regularized MDPs.} 

{ %
As outlined in \citet{levine2018reinforcement}, to solve \eqref{eq:original_goal}, various well-known algorithms exist in the literature on RL. The abovementioned concepts are useful in constructing these algorithms. These include policy gradients, which gradually optimize a policy using a policy neural network; soft Q-learning algorithms, which utilize the soft-Bellman equation and approximate the soft-value function with a value neural network; and soft actor-critic algorithms that leverage both policy and value neural networks. We will explore how these algorithms can be applied in the context of diffusion models shortly in \pref{sec:sum_planning} and \ref{sec:sum_planning_conservative}.}

\section{Fine-Tuning Diffusion Models with RL in Entropy Regularized MDPs} \label{sec:entropy}

 { In this section, as done in \citet{fan2023dpok,black2023training,uehara2024bridging}, we illustrate how fine-tuning can be formulated as an RL problem in soft-entropy regularized MDPs, where each denoising step of diffusion models corresponds to a policy in RL. Finally, we outline a specific RL problem of interest in our context. }

\begin{figure}[!th]
\centering
\begin{tikzpicture}[%
>=latex',node distance=2.5cm, minimum height=1.5cm, minimum width=1.5cm,
state/.style={draw, shape=circle, draw=orange, fill=orange!10, line width=0.5pt},
state-emp/.style={draw, shape=rectangle, draw=red, fill=red!10, line width=0.5pt,opacity=0.0},
reward/.style={draw, shape=rectangle, draw=blue, fill=blue!10, line width=0.5pt}
]
\node[state] (xT) at (0,0) {$x_{T}$};
\node[state,right of=xT] (xT1) {$x_{T-1}$};
\node[state,right of=xT1] (xT2) {$x_{T-2}$};
\node[state-emp,right of=xT2] (xT3) {...};
\node[state] (x1) at (10,0) {$x_{1}$};
\node[state,right of=x1] (x0) {$x_{0}$};
\node[reward,right of=x0] (r0) {$r$};
\draw[->] (xT) -- (xT1);
\node[text width=2cm] at (1.8,0.2) {$p_T$};
\draw[->] (xT1) -- (xT2);
\node[text width=2cm] at (4.3,0.2) {$p_{T-1}$};
\draw[->] (xT2) -- (xT3);
\node[text width=2cm] at (6.8,0.2) {$p_{T-2}$};
\draw[->] (xT3) -- (x1);
\node[text width=2cm] at (9.5,0.2) {$p_{2}$};
\draw[->] (x1) -- (x0);
\node[text width=2cm] at (12.0,0.2) {$p_{1}$};
\draw[dashed,->] (x0) -- (r0);
\end{tikzpicture}
\caption{Formulating fine-tuning in diffusion models using MDPs.}
\end{figure}
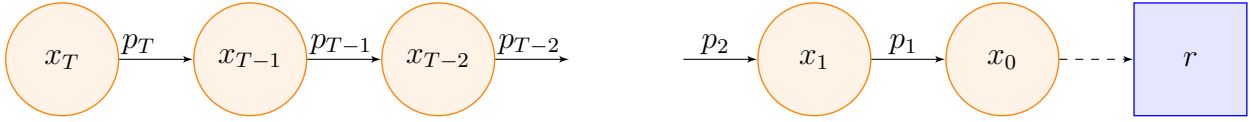

{To cast fine-tuning diffusion models as an RL problem}, we start with defining the following MDP:
\begin{itemize}
    \item The state space $\Scal$ and action space $\Acal$ correspond to the input space $\Xcal$.
    \item The transition dynamics at time $t$ (i.e., $P_t$) is an identity map $\delta(s_{t+1}={a_t})$.
    \item The reward at time $t \in[0,\cdots,T]$ (i.e., $r_t$) is provided only at $T$ as $r$ (down-stream reward function); but $0$ at other time steps. 
    \item The policy at time $t$ (i.e, $\pi_t$) corresponds to $p_{T+1-t}:\Xcal \to \Delta(\Xcal)$.
    \item The initial distribution at time $0$ corresponds to $p_{T+1}\in \Delta(\Xcal)$. With slight abuse of notation, we often denote it by $p_{T+1}(\cdot|\cdot)$, while this is just $p_{T+1}(\cdot)$. 
    \item The reference policy at $t$ (i.e., $\pi'_t$) corresponds to a denoising process in the pre-trained model $p^{\pre}_{T+1-t}$. 
\end{itemize} 
We list several things to note.
\begin{itemize}
    \item We reverse the time-evolving process to adhere to the standard notation in diffusion models, i.e., from $t=T$ to $t=0$. Hence, $s_t$ in standard MDPs corresponds to $x_{T+1-t}$ in diffusion models. 
    \item In our context, unlike standard RL scenarios, the transition dynamics are known.
\end{itemize}

\paragraph{Key RL Problem.}
Now, by reformulating the original objective of standard RL into our contexts, the objective function in \eqref{eq:original_goal} reduces to the following: 
{\color{red}\begin{align}\label{eq:key_plnanning}
       \{p^{\star}_t\}_t=   \argmax_{\{p_t \in [\RR^d \to \Delta(\RR^d)] \}_{t=T+1}^1 }\underbrace{\EE_{\{p_t\}}[r(x_0)]}_{\text{Reward}}   - \alpha \underbrace{\Sigma_{t=T+1}^1  \EE_{\{p_t\}}[\mathrm{KL}(p_t(\cdot|x_t)\|p^{\pre}_t (\cdot |x_t))  ]}_{\text{KL\,penalty} }
\end{align}
} 
where the expectation $\EE_{\{p_t\} }[\cdot]$ is taken with respect to $\prod_{t=T+1}^1 p_t(x_{t-1}|x_{t})$, i.e., $x_{T} \sim p_{T+1}(\cdot),x_{T-1}\sim p_{T-1}(\cdot \mid x_{T-1}),x_{T-2}\sim p_{T-2}(\cdot \mid x_{T-2}),\cdots $. In this article, we set this as an objective function in fine-tuning diffusion models. { This objective is natural as it seeks to optimize sequential denoising processes to maximize downstream rewards while maintaining proximity to pre-trained models. Subsequently, we investigate several algorithms to solve \eqref{eq:key_plnanning}. Before discussing these algorithms, we summarize several key theoretical properties that will aid their derivation. }

\section{Theory of RL-Based Fine-Tuning}   \label{sec:theory_RL}

{ So far, we have introduced a certain RL problem (i.e., \eqref{eq:key_plnanning}) as a fine-tuning diffusion model. In this section, we explain that solving this RL problem allows us to achieve the target distribution discussed in \pref{sec:formalized_goal}. Additionally, we present several important theoretical properties, such as the analytical form of marginal distributions and posterior distributions induced by fine-tuned models. This formulation is also instrumental in introducing several algorithms (reward-weighted MLE, value-weighted sampling, and path consistency learning in \pref{sec:sum_planning_conservative}), and establishing connections with related areas (classifier guidance in \pref{sec:connection}, and flow-based diffusion models in \pref{sec:flow_matching}). We start with several key concepts.}

\subsection{Key Concepts: Soft Value functions and Soft Bellman Equations.}

{ Now, reflecting on how soft optimal policies are expressed using soft value functions in \pref{sec:overview} in the context of standard RL problems, we derive several important concepts applicable to fine-tuning diffusion models. These concepts are later useful in constructing algorithms to solve our RL problem \eqref{eq:key_plnanning}.  }  

Firstly, as we see in \eqref{eq:sample}, soft-optimal policies are characterized as:
{\color{red}\begin{align}\label{eq:diffusion_soft}
      p^{\star}_{t}(\cdot |x_t)= \frac{\exp( v_{t-1}(\cdot )/\alpha) p^{\pre}_t(\cdot \mid x_t) }{\int \exp( v_{t-1}(x_{t-1} )/\alpha) p^{\pre}_t(x_{t-1} \mid x_t) d x_{t-1}  }
\end{align}
} 
where soft-value functions are defined as 
\begin{align*}
    v_{t}(x_t) &=  \EE_{\{p^{\star}_t \}}[r(x_0) - \alpha  \sum_{k=t}^1 \mathrm{KL}(p_k(\cdot|x_k)\|p^{\pre}_k (\cdot |x_k))  |x_t ], \\ 
    q_{t}(x_t,x_{t-1}) &=  \EE_{\{p^{\star}_t \}}[r(x_0) - \alpha  \sum_{k=t+1}^1 \mathrm{KL}(p_k(\cdot|x_k)\|p^{\pre}_k (\cdot |x_k))  |x_t,x_{t-1}]= v_{t-1}(x_{t-1}).  
 \end{align*}

Secondly, as we see in \eqref{eq:soft_bellman}, the soft-value functions are also recursively defined by the soft Bellman equations: 
\begin{align}\label{eq:soft_bellman_diffusion}
\begin{cases}
 \exp\left (\frac{ v_{t}(x_{t})}{\alpha} \right )= \int \exp\left (\frac{ v_{t-1}(x_{t-1})}{\alpha}\right)p^{\pre}_{t}(x_{t-1} \mid x_{t}) d x_{t-1}\,(t=T+1,\cdots,1),   \\ 
v_0(x_0)= r(x_0). 
\end{cases}
\end{align}
Now substituting the above in \eqref{eq:diffusion_soft}, we obtain 
\begin{align*}
      p^{\star}_{t}(\cdot |x_t)=  \frac{\exp( v_{t-1}(\cdot )/\alpha) p^{\pre}_t(\cdot \mid x_t) }{ \exp( v_{t}(x_t )/\alpha)  }. 
\end{align*}
{As mentioned earlier, these soft value functions and their recursive form will later serve as the basis for constructing several concrete fine-tuning algorithms (such as reward-weighted MLE and value-weighted sampling in  \pref{sec:sum_planning_conservative}).  }

\subsection{Induced Distributions after Fine-Tuning}

Now, with the above preparation, we can show that the induced distribution derived by this soft optimal policy is actually equal to our target distribution.

\begin{theorem}[Theorem 1 in \citet{uehara2024bridging}] \label{thm:main}
Let $p^{\star}(\cdot)$ be an induced distribution at time $0$ from optimal policies $\{p^{\star}_t\}_{t=T+1}^1$, i.e., $p^{\star}(x_0) = \int \{\prod_{t=T+1}^1 p^{\star}_t(x_{t-1}|x_t)\}d x_{1:T}$. The distribution $p^{\star}$ is equal to the target distribution~\eqref{eq:original_target2}, i.e.,   
\begin{align*}
    p^{\star}(x)= p_r(x).  
\end{align*}
\end{theorem}

This theorem states that after solving \eqref{eq:key_plnanning} and obtaining soft optimal policies, we can sample from the target distribution by sequentially running policies from $p^{\star}_{T+1}$ to $p^{\star}_1$. Thus, \eqref{eq:key_plnanning} serves as a natural objective for fine-tuning. { This fact is also useful in deriving a connection with classifier guidance in \pref{sec:connection}. } 

\paragraph{Marginal distributions.} We can derive the marginal distribution as follows. 

\begin{theorem}[Theorem 2 in \citet{uehara2024bridging} ]\label{thm:key2}
Let $p^{\star}_t(x_t)$ be the marginal distributions at $t$ induced by soft-optimal policies $\{p^{\star}_t\}_{t=T+1}^1$, i.e., $p^{\star}_t(x_t) = \int \{\prod_{k=T+1}^{t+1} p^{\star}_k(x_{k-1}|x_k)\}d x_{t+1:T}$.
Then, 
\begin{align*}
    p^{\star}_t(x_t) =  \frac{\exp(v_{t}(x_{t})/\alpha)p^{\pre}_{t}(x_{t})}{C},
\end{align*}
where $v_t(\cdot)$ is the soft-value function. %
\end{theorem} 
Interestingly, the normalizing constant is independent of $t$ in the above theorem.

\paragraph{Posterior distributions.} We can derive the posterior distribution as follows.

\begin{theorem}[Theorem 3 in \citet{uehara2024bridging}] \label{thm:key3}
Denote the posterior distribution of $x_t$ given $x_{t-1}$ for the distribution induced by soft-optimal policies $\{p^{\star}_t\}_{t=T+1}^1$ by $\{p^{\star}\}^b(\cdot \mid \cdot)$. We define the analogous objective for a pre-trained policy and denote it by  $p^{\pre}_t(\cdot \mid \cdot)$. Then, we get 
\begin{align*}
    \{p^{\star}\}^b_t(x_{t}|x_{t-1}) = p^{\pre}_t(x_{t}|x_{t-1}). 
\end{align*}
\end{theorem}
This theorem indicates that after solving \eqref{eq:key_plnanning}, the posterior distribution induced by pre-trained models remains preserved. This property plays an important role in constructing PCL (path consistency learning) in \pref{sec:gflownet}. 

\begin{remark}[Continuous-time formulation]
For simplicity, our explanation is generally based on the discrete-time formulation. However, as training of diffusion models could be formulated in the continuous-time formulation \citep{song2021maximum}, we can still extend most of our discussion of fine-tuning in our tutorial in the continuous-time formulation. For example, the above Theorems are extended to the continuous time formulation in \citet{uehara2024finetuning}. 
\end{remark}

{\section{RL-Based Fine-Tuning Algorithms 1: Non-Distribution-Constrained Approaches}} \label{sec:sum_planning}

{So far, we have explained how to frame fine-tuning diffusion models as the RL problem in entropy-regularized MDPs. Moving forward, we summarize actual algorithms that can solve the RL problem of interest described by  Equation~\eqref{eq:key_plnanning}. In this section, we introduce two algorithms: PPO and direct reward backpropagation. 

These algorithms are originally designed to optimize reward functions directly, meaning they operate effectively even without entropy regularization (i.e., $\alpha=0$). Consequently, they are well-suited for generating samples with high rewards that may not be in the original training dataset. More distribution-constrained algorithms that align closely with pre-trained models will be discussed in the subsequent section (\pref{sec:sum_planning}). Therefore, we classify algorithms in this section (i.e., PPO and direct reward backpropagation) as non-distribution-constrained approaches. The whole summary is described in Table~\ref{tab:summary_algorithm}.       }

\begin{table}[!t]
    \centering
     \caption{Description of each RL algorithm for fine-tuning diffusion models (note that value-weighted sampling is technically not a fine-tuning algorithm.) Note (1) ``Without learning value functions'' refers to the capability of algorithms to directly utilize non-differentiable black-box reward feedback, bypassing the necessity to train differentiable reward functions, (2) { ``Distribution-constrained'' indicates that the algorithms are designed to maintain proximity to pre-trained models. Based on it, the practical recommendation of algorithms is summarized in \pref{fig:practical}.   } 
     }
    \label{tab:summary_algorithm}
    \begin{tabular}{p{0.3\textwidth}p{0.10\textwidth}p{0.15\textwidth}p{0.17\textwidth}p{0.10\textwidth}  } \toprule 
         &      Memory efficiency & Computational efficiency & Without learning value functions & Distribution-constrained \\   \midrule
   Soft PPO  &  $\checkmark$ &   &  $\checkmark$    &     \\  
      Reward backpropagation      &   & $\checkmark$  & $\checkmark$   &  \\
   Reward-weighted MLE & $\checkmark$  &  &  $\checkmark$ & $\checkmark$ \\
   Value-weighted sampling  & $\checkmark$ &  &  & $\checkmark$  \\ 
   Path consistency learning &  $\checkmark$ &   &  &   $\checkmark$  \\ \bottomrule 
    \end{tabular}
\end{table}

\subsection{Soft Proximal Policy Optimization (PPO) }\label{sec:PPO}

In order to solve Equation~\eqref{eq:key_plnanning}, \citet{fan2023dpok,clark2023directly} propose using PPO \citep{schulman2017proximal}. PPO has been widely used in RL, as well as, in the literature in fine-tuning LLMs, due to its stability and simplicity. In the standard context of RL, this is especially preferred over Q-learning when the action space is high-dimensional. 

The PPO algorithm is described in \pref{alg:ppo}. This is an iterative procedure of updating a parameter $\theta$. Each iteration comprises two steps: firstly, samples are generated by executing current policies to construct the loss function (inspired by policy gradient formulation); secondly, the parameter $\theta$ is updated by computing the gradient of the loss function. 

PPO offers several advantages. The approach is known for its stability and relatively straightforward implementation. Stability comes from the conservative parameter updates. Indeed, PPO builds upon TRPO \citep{schulman2015trust}, where parameters are conservatively updated with a KL penalty term (between $\theta_{s+1}$ and $\theta_s$) to prevent significant deviation from the current parameter. This gives us stability in the optimization landscape. Furthermore, in \pref{alg:ppo}, we do not necessarily need to rely on value functions, although they could be useful for variance reduction. As discussed in the subsequent subsection, this can be advantageous compared to other methods, especially since learning value functions can be challenging in high-dimensional spaces, particularly within the context of diffusion models.

\begin{algorithm}[!t]
\caption{Soft PPO}\label{alg:ppo}
\begin{algorithmic}[1]
     \STATE {\bf Require}: Pre-trained model $\{\Ncal(\rho(x_{t},t;\theta_{\pre}),\sigma^2(t))\}_{t=T+1}^1$, batch size $m$, parameter $\alpha \in \RR^+$, learning rate $\gamma$ 
     \STATE {\bf Initialize}: $\theta_1 = \theta_{\pre}$
     \FOR{$s \in [1,\cdots,S]$}
      \STATE Collect $m$ samples $\{x^{(i)}_t(\theta) \}_{t=T}^0$ from a current diffusion model (i.e., generating by sequentially running polices $\{\Ncal(\rho(x_{t},t;\theta),\sigma^2(t))\}_{t=T+1}^1$ from $t=T+1$ to $t=1$) 
      \STATE Update as follows (several times if needed): 
      {\small 
   \begin{align} \label{eq:natural}
    \theta_{s+1} \leftarrow \theta_s -  \frac{\gamma}{m}\nabla_{\theta}\sum_{t=T+1}^1 \sum_{i=1}^m &  \left[ \min\left \{\tilde r_t(x^{(i)}_0,x^{(i)}_{t})\frac{p(x^{(i)}_{t-1}|x^{(i)}_{t};\theta)}{p(x^{(i)}_{t-1}|x^{(i)}_{t};\theta_{s})}, \right.  \right.  \\
     &  \left. \left. \tilde r_t(x^{(i)}_0,x^{(i)}_t) \cdot \mathrm{Clip}\left (\frac{p(x^{(i)}_{t-1}|x^{(i)}_{t};\theta)}{p(x^{(i)}_{t-1}|x^{(i)}_{t};\theta_{s})}, 1-\epsilon,1+\epsilon\right ) \right \} \right] |_{\theta = \theta_s},  \nonumber 
 \end{align}
 where 
    \begin{align*} 
   \tilde r_t(x_0,x_t) := - r(x_0) +  \underbrace{\alpha \frac{\|\rho (x_t,t;\theta) -\rho (x_t,t;\theta^{\pre}) \|^2 }{2 \sigma^2(t)}}_{\text{KL term}}. \nonumber 
\end{align*}
} 
      \ENDFOR
  \STATE {\bf Output}: Policy $\{p_t(\cdot \mid \cdot; \theta_S) \}_{t=T+1}^1$   
\end{algorithmic}
\end{algorithm}

\subsection{Direct Reward Backpropagation}\label{sec:reward_prop}

Another standard approach is a differentiable optimization \citep{clark2023directly,prabhudesai2023aligning,uehara2024finetuning}, where gradients are directly propagated from reward functions to update policies. 

The entire algorithm is detailed in \pref{alg:main}. This reward backpropagation entails an iterative process of updating a parameter $\theta$. Each iteration comprises two steps; firstly, samples are generated by executing current policies to approximate the expectation in the loss function, which is directly derived from \eqref{eq:key_plnanning}; second, the current parameter $\theta$ is updated by computing the gradient of the loss function.

\paragraph{Advantages over PPO.} This approach offers further simplicity in implementation in a case where we already have a pre-trained differentiable reward model. Furthermore, the training speed is much faster since we are directly back-propagating from rewards. 

\paragraph{Potential disadvantages over PPO.} Reward backpropagation may face memory inefficiency issues. However, there are strategies to mitigate this challenge. Firstly, implementing gradient accumulation can help to keep a large batch size.
Secondly, as proposed in DRaFT \citep{clark2023directly}, propagating rewards backward from time $0$ to $k$ ($k$ is an intermediate step smaller than $T$) and updating policies from $k$ to $0$ can still yield high performance.   
Thirdly, drawing from insights in the literature on neural SDE/ODE \citep{chen2018neural}, more memory-efficient advanced techniques such as adjoint methods could be helpful. 

Another potential drawback is the requirement for  ``differentiable'' reward functions. Often, reward functions are obtained in a non-differentiable black-box way (e.g., computational feedback derived from physical simulations). In such scenarios, using direct backpropagation necessitates the learning of differentiable reward functions even if accurate reward feedback is available. This learning step can pose challenges as it involves data collection and constructing suitable reward models.  

\begin{algorithm}[!t]
\caption{Reward backpropagation }\label{alg:main}
\begin{algorithmic}[1]
     \STATE {\bf Require}: Pre-trained model $\{\Ncal(\rho(x_{t},t;\theta_{\pre}),\sigma^2(t))\}_{t=T+1}^1$, batch size $m$, parameter $\alpha \in \RR^+$, learning rate $\gamma$
     \STATE \emph{Train a differentiable reward function (if reward feedback is not differentiable)} 
     \STATE {\bf Initialize}: $\theta_1 = \theta_{\pre}$
     \FOR{$s \in [1,\cdots,S]$}
      \STATE Collect $m$ samples $\{x^{(i)}_t(\theta) \}_{t=T}^0$ from a current diffusion model (i.e., generating by sequentially running polices $\{\Ncal(\rho(x_{t},t;\theta),\sigma^2(t))\}_{t=T+1}^1$ from $t=T+1$ to $t=1$) 
      \STATE Update $\theta_s$ to $\theta_{s+1}$: 
      \begin{align}\label{eq:key_reward}
      \theta_{s+1} \leftarrow \theta_s + \gamma \left [ \frac{1}{m} \sum_{i=1}^m \left \{  r(x^{(i)}_0(\theta) )  -  \alpha     \sum_{t=T+1}^1 \frac{\|\rho(x^{(i)}_t(\theta),t;\theta)-\rho(x^{(i)}_t(\theta),t; \theta_{\pre}) \|^2 }{2\sigma^2(t) } \right \} \right ] |_{\theta = \theta_s} . 
      \end{align}  
      \ENDFOR
  \STATE {\bf Output}: Policy $\{p_t(\cdot \mid \cdot; \theta_S) \}_{t=T+1}^1$   
\end{algorithmic}
\end{algorithm}

\section{RL-Based Fine-Tuning Algorithms 2: Distribution-Constrained Approaches}\label{sec:sum_planning_conservative}

{ In this section, following \pref{sec:sum_planning}, we present three additional algorithms (reward-weighted MLE, value-weighted sampling, and path consistency learning) aimed at solving the RL problem of interest defined by Equation~\eqref{eq:key_plnanning}. In the context of standard RL, these algorithms are tailored to align closely with reference policies, specifically pre-trained diffusion models in our context. Formally, indeed, all algorithms in this section are not well-defined when $\alpha = 0$ (i.e., without entropy regularization). Hence, we categorize these three algorithms as distribution-constrained approaches.

The algorithms in this section excel in preserving the characteristics of pre-trained diffusion models. Practically, this property becomes especially crucial when reward functions are learned from training data, and we want to avoid being fooled by distribution samples (a.k.a. overoptimization as detailed in \pref{sec:unknown}). However, as a caveat, this property might also pose challenges in effectively generating high-reward samples beyond the training data. This implies that these approaches may not be suitable when accurate reward feedback is readily available without learning. Hence, we generally recommend using them when reward functions are unknown. } 

\subsection{Reward-Weighted MLE}\label{sec:reward-weighted}

{ Here, we elucidate an approach based on reward-weighted MLE \citep{peters2010relative}, a technique commonly employed in offline RL \citep{peng2019advantage}. While \citet[Algorithm 2]{fan2023dpok} and \citet{zhang2023towards} propose variations of reward-weighted MLE for diffusion models, the specific formulation of reward-weighted MLE discussed here does not seem to have been explicitly detailed previously. 
Therefore, unlike the previous section, we start by outlining the detailed rationale for this approach. Subsequently, we provide a comprehensive explanation of the algorithm. Finally, we delve into its connection with the original training loss of diffusion models.} 

\paragraph{Motivation.} First, from \pref{thm:key2}, recall the form of the optimal policy $p^{\star}_t(x_{t-1}|x_t)$:
\begin{align*}
    p^{\star}_t(x_{t-1}|x_t) :=  \frac{\exp(v_{t-1}(x_{t-1})/\alpha)p^{\pre}_{t}(x_{t-1}|x_{t})}{\exp(v_t(x_t)/\alpha) }.
\end{align*}
Now, we have:
\begin{align*}
   p^{\star}_t =   \argmin_{p_t:\Xcal \to \Delta(\Xcal) } \EE_{x_t \sim u_t }[\KL (p^{\star}_t(\cdot|x_t) \|  p_t(\cdot|x_t))  ] 
\end{align*}
where $u_t \in \Delta(\Xcal)$ is a roll-in distribution encompassing the entire space $\Xcal$
This can be reformulated as value-weighted MLE as follows.

\begin{lemma}[Value-weighted MLE]\label{lem:value}
When $\Pi_t =[\Xcal \to \Delta(\Xcal)] $, the policy $p^{\star}_t$ is equal to   
    \begin{align*}
 p^{\star}_t(\cdot|\cdot) = \argmax_{p_t \in \Pi_t} \EE_{x_{t-1}\sim p^{\pre}_{t-1}(\cdot \mid x_{t}),x_t\sim u_t }\left[\exp\left( \frac{v_{t-1}(x_{t-1})}{\alpha}\right) \log p_t(x_{t-1}|x_t) \right ]. 
\end{align*}
\end{lemma}

This lemma illustrates that if $v_{t-1}$ is known, $p^{\star}_t$ can be estimated using weighted maximum likelihood estimation (MLE). While this formulation is commonly used in standard RL (Peng et al., 2019), in our context of fine-tuning diffusion models, learning a value function is often challenging. Interestingly, this reward-weighted MLE can be performed without directly estimating the soft value function after proper reformulation. To demonstrate this, let's utilize the following lemma:
\begin{lemma}[Characterization of soft optimal value functions]\label{lem:characterization}
    \begin{align*}
   \exp\left(\frac{v_t(x_t)}{\alpha} \right) = \EE_{x_0\sim p^{\pre}_1(x_1),\cdots, x_{t-1}\sim p^{\pre}_{t-1}(x_t) }\left[\exp\left(\frac{r(x_0)}{\alpha} \right )|x_t \right ]. 
\end{align*}
Recall $\EE_{\{p^{\pre}_t\}}[\cdot|x_t]$ means $\EE_{x_0\sim p^{\pre}_1(x_1),\cdots, x_{t-1}\sim p^{\pre}_{t-1}(x_t) }[\cdot|x_t]$. 
\end{lemma}

\begin{proof}
This is obtained by recursively using the soft-Bellman equation \eqref{eq:soft_bellman_diffusion}: 
\begin{align}
 \exp\left (\frac{ v_{t}(x_{t})}{\alpha}\right )& = \int \exp\left (\frac{ v_{t-1}(x_{t-1})}{\alpha}\right)p^{\pre}_{t}(x_{t-1} \mid x_{t}) d x_{t-1}=\cdots = \EE_{ \{p^{\pre}_t\} }\left[\exp\left(\frac{r(x_0)}{\alpha} \right )|x_t \right ]. 
\end{align}
\end{proof}

\paragraph{Algorithm.}

\begin{algorithm}[!t]
\caption{Reward-weighed MLE}\label{alg:weightedMLE}
\begin{algorithmic}[1]
     \STATE {\bf Require}: Pre-trained model $\{\Ncal(\rho(x_{t},t;\theta_{\pre}),\sigma^2(t))\}_{t=T+1}^1$,  batch size $m$, parameter $\alpha \in \RR^+$, learning rate $\gamma$
      \STATE {\bf Initialize}: $\theta_1 = \theta_{\pre}$
     \FOR{$s \in [1,\cdots,S]$}
    \FOR{$k \in [T+1,\cdots,1]$}
      \STATE Collect $m$ samples $\{x^{(i,t)}_k\}_{i=1,k=T}^{m,0}$ from \\  a policy $p_{T+1}(\cdot \mid \cdot; \theta_s)  ,\cdots,p_{t+1}(\cdot| \cdot; \theta_s),p^{\pre}_{t}(\cdot |\cdot),\cdots,p^{\pre}_{1}(\cdot | \cdot) $.  
     \ENDFOR
      \STATE Update $\theta_s$ to $\theta_{s+1}$ as follows: 
      \begin{align}\label{eq:key}
\theta_{s+1} \leftarrow \theta_s - \gamma \nabla_{\theta} \sum_{t=T+1}^1 \sum_{i=1}^m  \left[\exp \left (\frac{r(x^{(i,t)}_0)}{\alpha} \right) \frac{\|x^{(i,t)}_{t-1} -\rho(x^{(i,t)}_t,t;\theta) \|^2_2}{\{\sigma^{\diamond}_t\}^2 }   \right]|_{\theta=\theta_s} . 
      \end{align} 
        \ENDFOR
  \STATE {\bf Output}: Policy $\{p_t(\cdot \mid \cdot; \theta_S) \}_{t=T+1}^1$   
\end{algorithmic}
\end{algorithm}

Now, we are ready to present the algorithm. By combining \pref{lem:value} and \pref{lem:characterization}, we obtain the following.

\begin{lemma}[Reward-weighted MLE]\label{lem:reward-weighted}
When $\Pi_t =[\Xcal \to \Delta(\Xcal)] $, 
\begin{align}\label{eq:monte_carlo}
 p^{\star}_t = \argmax_{p_t \in \Pi_t} \EE_{x_0 \sim p^{\pre}_1(x_1),\cdots, x_{t-1}\sim p^{\pre}_{t-1}(x_{t}),x_t\sim u_t }\left[\exp\left(\frac{r(x_0)}{\alpha}\right) \log p_t(x_{t-1}|x_t) \right ], 
\end{align}
\end{lemma}
\begin{proof}
Using  \pref{lem:characterization}, we have 
\begin{align*}
    & \EE_{x_{t-1}\sim p^{\pre}_{t-1}(\cdot \mid x_{t}),x_t\sim u_t }\left[\exp\left( \frac{v_{t-1}(x_{t-1})}{\alpha}\right) \log p_t(x_{t-1}|x_t) \right ] \\
    & = \EE_{x_{t-1}\sim p^{\pre}_{t-1}(\cdot \mid x_{t}),x_t\sim u_t }\left[ \EE_{ \{p^{\pre}_t\} }\left[\exp\left(\frac{r(x_0)}{\alpha} \right )|x_{t-1} \right ] \log p_t(x_{t-1}|x_t) \right ] \\
    & = \EE_{x_0 \sim p^{\pre}_1(x_1),\cdots, x_{t-1}\sim p^{\pre}_{t-1}(x_{t}),x_t\sim u_t }\left[\exp\left(\frac{r(x_0)}{\alpha}\right) \log p_t(x_{t-1}|x_t) \right ]. 
\end{align*}
The rest of the proof is obvious by using \pref{lem:value}. 
\end{proof}

Then, after approximating the expectation in \eqref{eq:monte_carlo}, by using a Gaussian policy class with the mean parameterized by neural networks as a policy class $\Pi_t$, we can estimate $p^{\star}_t$. Finally, the entire algorithm is described in Algorithm~\ref{alg:weightedMLE}.  

{Here, we give two remarks. This is an off-policy algorithm. Hence, we can use any roll-in policies as $u_t$ in \pref{lem:reward-weighted}. In Algorithm~\ref{alg:weightedMLE}, we use the current policy as a roll-in policy. Additionally, in Algorithm~\ref{alg:weightedMLE}, the loss function \eqref{eq:key} is derived by recalling that, up to constant, $ -\log p_t(x_{t-1}|x_t)$ is equal to 
$\frac{\|x_{t-1} -\rho(x_t,t;\theta) \|^2_2}{\{\sigma^{\diamond}_t\}^2}. $
} 

{ \paragraph{Advantages and potential disadvantages.} Like PPO in \pref{sec:PPO}, this approach is expected to be memory efficient, and does not require learning differentiable reward functions, which can often be challenging. However, compared to direct reward backpropagation Section~\ref{sec:reward_prop}, it might not be as computationally efficient. Furthermore, unlike PPO, this algorithm is not effective when $\alpha=0$, potentially limiting its ability to generate samples with extremely high rewards.}   

{\subsubsection{Relation with Loss Functions for the Original Training Objective}

In this subsection, we explore the connection with the original loss function of pre-trained diffusion models. To see that, as we see in \pref{sec:training}, recall  
\begin{align*}
    x^{(i,t)}_{t-1}= \underbrace{x_t^{(i,t)} + [0.5x_t - 1/\sigma^{\diamond}_t \epsilon_{\bar \theta_{\pre}}(x_t,T-t) ](\delta t)}_{\rho(x_t^{(i,t)},t;\theta_{\pre}) }+\epsilon^{(i,t)}_t \sigma^{\diamond}_t ,\,\,\epsilon^{(i,t)}\sim \Ncal(0,\mathbb{I}).  
\end{align*}
Then, the loss function \eqref{eq:key} in reward-weighted MLE reduces to 
 \begin{align}
 \theta_s - \gamma \nabla_{\theta} \sum_{t=T+1}^1 \sum_{i=1}^m  \left[\exp \left (\frac{r(x^{(i,t)}_0)}{\alpha} \right) \left \|\epsilon^{(i,t)}_t -\frac{\epsilon(x^{(i,t)}_t,T-t;\bar \theta_{\pre})-\epsilon(x^{(i,t)}_t,T-t;\theta) }{ \{\sigma^{\diamond}_t\}^2 } \right \|^2_2   \right]|_{\theta=\theta_s} . 
\end{align} 
This objective function closely resembles the reward-weighted version of the loss function \eqref{eq:loss_score} used for training pre-trained diffusion models.
}

\subsection{Value-Weighted Sampling}\label{sec:value}

{ Thus far, we have discussed methods for fine-tuning pre-trained diffusion models. Now, let's delve into an alternative approach during inference that aims to sample from the target distribution $p_r$ without explicitly fine-tuning the diffusion models. In essence, this approach involves incorporating gradients of value functions during inference alongside the denoising process in pre-trained diffusion models. Hence, we refer to this approach as value-weighted sampling. While it seems that this method has not been explicitly formalized in previous literature, the value-weighted sampling closely connects with classifier guidance, as discussed in \pref{sec:classfier-based}.   } 

{ Before delving into the algorithmic details, we outline the motivation. Subsequently, we present the concrete algorithm and discuss its advantages—specifically, its capability to operate without the necessity of fine-tuning diffusion models—and its disadvantage, which involves the need to obtain differentiable value functions.
}

\paragraph{Motivation.} Considering a Gaussian policy $x_{t-1}\sim \mathcal{N}(\tilde{\rho}(x_t,t),\sigma^2(t))$, we aim to determine $\tilde{\rho}(x_t,t;\theta)$ such that $\mathcal{N}(\tilde{\rho}(x_t,t;\theta),\sigma^2(t))$ closely approximates $p^{\star}_t(\cdot|x_t)$. Here, typically, we have $\rho(x_t,t;\theta) = x_t + (\delta t)\bar g(x_{t},t)$ and $\sigma^2(t) = \tilde g(t)(\delta t)$ for certain function $\bar g:\Xcal \times [0,T]\to \RR^d$ and $\tilde g:[0,T] \to \RR$, as we have explained in \pref{sec:training}. Then, using $\propto$ as equality up to the normalizing constant,  
\begin{align*}
     p^{\star}_t(x_{t-1}\mid x_t) &\propto  \exp(v_{t-1}(x_{t-1})/\alpha)\exp\left (-\frac{\|x_{t-1}-x_t- (\delta t)\bar g(x_{t},t) \|^2 }{0.5 \tilde g(t)(\delta t)  } \right) \\ 
    & \approx \exp(v_t(x_{t}) + \nabla v_{t}(x_{t})/\alpha \cdot  \{x_{t-1}-x_{t}\}   )\exp\left (-\frac{\|x_{t-1}-x_t-  (\delta t)\bar g(x_{t},t)\|^2 }{0.5 \tilde g(t)(\delta t) } \right) \\
    &\approx  \exp\left (-\frac{ \| x_{t-1}-x_t- (\delta t) \tilde g(t) \nabla v_t(x_{t})/\alpha - (\delta t)\bar g(x_{t},t) )\| ^2}{0.5 \tilde g(t)(\delta t) } \right)  
\end{align*}
Hence, we can approximate:
\begin{align}\label{eq:sampling}
   \tilde \rho(x_t, t ) \approx \frac{\sigma^2(t) \nabla v_{t}(x_{t})}{\alpha} + \rho(x_t, t; \theta_{\pre}). 
\end{align}

\begin{remark}
Note while the above derivation is heuristic, it is formalized in \citet[Lemma 2]{uehara2024finetuning} in the continuous-time formulation.
\end{remark}

\begin{algorithm}[!t]
\caption{Value-weighted sampling}\label{alg:value_sampling}
\begin{algorithmic}[1]

 \STATE {\bf Require}:  Pre-trained model $\{p^{\pre}_{t}(x_{t-1}|x_t)\}_t = \{ \Ncal(\rho(x_t,t; \theta_{\pre}),\sigma^2(t))\}_t $. 
        \STATE Estimate $v:\Xcal \times [0,T] \to \RR$  and denote it by $\hat v(\cdot,\cdot)$
        \begin{itemize}
            \item (1) Monte-Carlo approach in \eqref{eq:estimate_value}
            \item (2) Value iteration approach (Soft Q-learning) in \eqref{eq:estimate_value2} in \pref{sec:soft-q-learning}
            \item (3) Approximation using Tweedie's formula in \pref{sec:tweedie}
        \end{itemize}
    \FOR{$t \in [T+1,\cdots,1]$}
  
      \STATE Set 
     \begin{align*}
        \tilde \rho(x_t,t ):= \frac{\sigma^2(t) \nabla_x \hat v(x,t)|_{x=x_t}}{\alpha}  + \rho(x_t,t; \theta_{\pre}). 
\end{align*}
\ENDFOR 
  \STATE {\bf Output}: $\{\Ncal(\tilde \rho(x_t,t),\sigma^2(t))\}_{t=T+1}^1$
\end{algorithmic}
\end{algorithm}

\paragraph{Algorithm.} Utilizing \eqref{eq:sampling}, the entire algorithm of value-weighted sampling is outlined in Algorithm~\ref{alg:value_sampling}. This method doesn't involve updating parameters in a diffusion model; hence, it's not a fine-tuning method. 

Note that in this algorithm, we require a form of $v_t(\cdot)$ to compute gradients. This value function can be estimated through regression using the characterization described in Lemma~\ref{lem:characterization}. Here, inspired by \pref{lem:characterization}, we use the following loss function: 
\begin{align}\label{eq:estimate_value}
   \hat v_t(x) = \argmin_{h: [\RR^d, [0,T]] \to \RR } \sum_{t= T+1}^1 \sum_{i=1}^m \left[\left \{ \exp\left(\frac{h_t(x^{(i,t)}_t)}{\alpha} \right) - \exp\left(\frac{r(x^{(i,t)}_0)}{\alpha} \right ) \right \}^2  \right ], 
\end{align}
where the data is collected in an off-policy way. {Later, we explain two alternative approaches in \pref{sec:soft-q-learning} and \pref{sec:tweedie}. }

\paragraph{Advantages and potential disadvantages.} {The value-weighted sampling is straightforward and less memory-intensive since it avoids fine-tuning.} Therefore, it presents an appealing simple option if it performs well. Indeed, in various inverse problems, such as inpainting, super-resolution, and colorization, setting $r(x)$ as a likelihood of the measurement model $y=g(x)+\epsilon$ (where $y$ represents actual measurements, $\epsilon$ denotes noise, and $g$ defines the measurement function) has proven highly successful \citep{chung2022improving,bansal2023universal,chung2022diffusion}. 

{ The potential drawback compared to fine-tuning algorithms so far (i.e., PPO and reward-weighted MLE) is the necessity to learn differetiable soft value functions like direct reward backpropagation. This learning process is often not straightforward as explained in \pref{sec:reward_prop}. The previously discussed fine-tuning algorithms (i.e., PPO, reward-weighted MLE) circumvent this requirement by obtaining rewards to go via Monte Carlo approaches, thereby avoiding the need for soft value functions.
}

\subsubsection{Soft Q-learning}\label{sec:soft-q-learning}

{ We have elucidated that leveraging Lemma~\ref{lem:characterization}, we can estimate soft value functions $v_t(\cdot)$ based on Equation~\eqref{eq:estimate_value} in a Monte Carlo way. Subsequently, these soft value functions are used in value-weighted sampling to sample from the target distribution $p_r$. Alternatively, there is another method that involves using soft Bellman equations to estimate soft value functions $v_t(\cdot)$. This technique is commonly called soft Q-learning in the context of standard RL. }

First, recalling the soft Bellman equations in  \eqref{eq:soft_bellman}, we have 
\begin{align*} 
\exp \left  (\frac{ v_{t}(x_{t})}{\alpha} \right )=  \int \exp\left (\frac{ v_{t-1}(x_{t-1})}{\alpha}\right)p^{\pre}_{t}(x_{t-1} \mid x_{t}) d x_{t-1}. 
\end{align*}
Taking the logarithm, we obtain 
\begin{align*} 
 v_{t}(x_{t}) = \alpha \log \int \exp\left (\frac{ v_{t-1}(x_{t-1})}{\alpha}\right)p^{\pre}_{t}(x_{t-1} \mid x_{t}) d x_{t-1}. 
\end{align*}
Hence,  
\begin{align*}
     v_t = \argmin_{h:\Xcal \to \RR}\EE_{x_t\sim u_t} \left[\left \{\frac{h(x_t)}{\alpha} -\log \int \exp\left( \frac{v_{t-1}(x_{t-1})}{\alpha} \right)p_{\pre}(x_{t-1}|x_t) d x_{t-1}  \right \}^2   \right ]
\end{align*}
where $u_t \in \Delta(\mathcal{X})$ is any roll-in distribution that covers the entire space $\mathcal{X}$. Using this relation and replacing the expectation $\EE_{x_t\sim u_t}$ with empirical approximation, we are able to estimate soft value functions $v_t$ in a recursive manner: 
\begin{align}\label{eq:estimate_value2}
  \{\hat v^{\langle j \rangle}_t(x)\}_t \leftarrow \argmin_{h: [\RR^d, [0,T]] \to \RR } \sum_{t= T+1}^1 \sum_{i=1}^m \left[ \left\{ h_t(x^{(i,t)}_t)  - \log \int \exp\left( \frac{\hat v^{\langle j-1 \rangle}_{t-1}(x_{t-1})}{\alpha} \right)p_{\pre}(x_{t-1}|x^{(i,t)}_t) d x_{t-1}  \right \}^2   \right], 
\end{align}
where the data is collected in an off-policy way.

\begin{remark}
Although soft Q-learning is widely used in standard RL \citep{schulman2017equivalence}, it cannot be directly applied to our fine-tuning context without resorting to value-weighted sampling or value-weighted MLE. This is because, even if we estimate soft-value functions as $\hat v_t$, substituting $v_t$ with $\hat v_t$ in the soft-optimal policy results in an unnormalized policy.  
\end{remark}

\subsubsection{Approximation using Tweedie's formula }\label{sec:tweedie}

So far, we have explained two approaches: a Monte Carlo approach and a value iteration approach to estimate soft value functions. However, learning value functions in \eqref{eq:estimate_value} can still be often challenging in practice. Therefore, we can employ approximation strategies inspired by recent literature on classifier guidance (e.g., reconstruction guidance \citep{ho2022video}, manifold constrained gradients \citep{chung2022improving}, universal guidance \citep{bansal2023universal}, and diffusion posterior sampling \citep{chung2022diffusion}). 

Specifically, we adopt the following approximation:
\begin{align}
     v_t(x_t) &=\alpha \log \EE_{ \{p^{\pre}_t\} }\left[\exp\left(\frac{r(x_0)}{\alpha} \right)  |x_t \right]= \alpha \log \left( \int  \exp\left(\frac{r(x_0)}{\alpha} \right)  p^{\pre}(x_0|x_t)dx_0  \right)  \label{eq:integral} \\ 
     & \approx \alpha \log \left(\exp\left(\frac{r( \hat x_0(x_t) )}{\alpha} \right)  \right), \quad \hat x_0(x_t)=\EE_{ \{p^{\pre}_t\} }[x_0\mid x_t]  \label{eq:integral2},  \\ 
     &= r(\hat x_0(x_t)).  \nonumber
\end{align}
Here, we replace the integration in \eqref{eq:integral} with a Dirac delta distribution with the posterior mean. Importantly, we can calculate $\hat x_0(x_t)=\EE_{ {p^{\pre}_t} }[x_0\mid x_t]$ using the pre-trained (score-based) diffusion model based on Tweedie's formula:
\begin{align*}
    \EE_{ {p^{\pre}_t} }[x_0\mid x_t]=\frac{x_t + \{\sigma^{\diamond}_t\}^2 \nabla  \log q_t(x_t) }{\mu^{\diamond}_t }. 
\end{align*}
Recall that the notation $\mu^{\diamond}_t, \sigma^{\diamond}_t,q_t$ are defined in \pref{eq:conditional}.  Finally, { by recalling $\nabla  \log q_t \approx s_{\hat \theta_{\pre}}(x_t,t)$ in score-based diffusion models}, we can approximate $\nabla v_t(x)$ with 
\begin{align*}
    \nabla_{x_t} r\left( \frac{x_t + \{\sigma^{\diamond}_t\}^2 \nabla  s_{\hat \theta_{\pre}}(x_t,t) }{\mu^{\diamond}_t } \right), 
\end{align*}
and plug it into \pref{alg:value_sampling} (i.e., value-weighted sampling). 

{As mentioned earlier, this approach has been practically widely used in the context of classifier guidance. Despite its simplicity, the approximation error can be significant, as it can not be diminished even with a large sample size because the discrepancy from \eqref{eq:integral} to \eqref{eq:integral2} is not merely a statistical error that diminishes with increasing sample size. }

\subsubsection{Zeroth-Order Guidance using Path Integral Control }\label{sec:path_integral}

In the previous subsection (\pref{sec:tweedie}), we discussed an approximation technique to bypass the necessity of learning explicit value functions. Another approach to bypass this requirement is to use control-based techniques for obtaining soft-optimal policies, a.k.a., path integral control \citep{theodorou2010generalized,williams2017model,kazim2024recent}. 

Recall that we have $\rho(x_t,t;\theta) = x_t + (\delta t)\bar g(x_{t},t)$ and $\sigma^2(t) = \tilde g(t)(\delta t)$, the motivation behind this approach is as follows. Initially, we have:
\begin{align*}
    &\quad \nabla_{x_t} \exp(v(x_t)/\alpha)\\
    &= \EE_{\{p^{\pre}_t\} }[\exp(v_{t-1}(x_{t-1})/\alpha)\mid x_t ] \tag{Soft Bellman equation} \\ 
    &=\nabla_{x_t} \left\{ \int \exp(v_{t-1}(x_{t-1}/\alpha))\exp\left (-\frac{(x_{t-1}-x_t- (\delta t)\bar g(x_{t},t) )^2 }{0.5 \tilde g(t)(\delta t)  } \right)d x_{t-1} \right\} \\
    &=\EE_{\{p^{\pre}_t\} }\left[\exp(v_{t-1}(x_{t-1}/\alpha))\{x_{t-1}-x_t - (\delta t)\bar g(x_t,t) \}\frac{\{ 1 + (\delta t)\nabla_{x_t} \bar g(x_t,t)\} }{ \tilde g(t)(\delta t)}    |x_t \right ]\\
    &\approx \frac{1}{\tilde g(t)(\delta t) }\EE_{\{p^{\pre}_t\} }\left[\exp(v_{t-1}(x_{t-1}/\alpha))\{x_{t-1}-x_t - (\delta t)\bar g(x_t,t) \}   |x_t \right ]\\
    &=\frac{1}{\tilde g(t)(\delta t) }\EE_{\{p^{\pre}_t\} }\left[\EE_{ \{p^{\pre}_t\}}\left[\exp(r(x_0)/\alpha) \mid x_{t-1} \right]\epsilon_t |x_t \right ]=  \frac{1}{\tilde g(t)(\delta t) }\EE_{\{p^{\pre}_t\} }\left[\exp(r(x_{0}/\alpha))\epsilon_t |x_t \right ]. 
\end{align*} 
Now, we obtain: 
\begin{align}\label{eq:path_integral}
    \frac{\sigma^2(t)\nabla_{x_t} v(x_t)}{\alpha} \approx  \frac{1}{\alpha } \times \frac{\EE_{\{p^{\pre}_t\} }\left[\exp(r(x_{0})/\alpha))\epsilon_t |x_t \right ] }{\EE_{\{p^{\pre}_t\} }\left[\exp(r(x_{0})/\alpha)) |x_t \right ]}. 
\end{align}
Importantly, this approach does not require any differentiation. Therefore, by running policies from pre-trained models and simply using Monte Carlo approximation in both the denominator and numerator, we can estimate the above quantity (the right-hand side of \eqref{eq:path_integral}) without making any models (classifiers), unlike classifier guidance. Note while this approach is widely used in the control community, it may not be feasible for diffusion models due to the high dimensionality of the input space.

\subsection{Path Consistency Learning (Losses Often Used in Gflownets)}\label{sec:gflownet}

\begin{algorithm}[!t]
\caption{Path Consistency Learning (Training with detailed balance loss)}\label{alg:detailed_balance_loss}
\begin{algorithmic}[1]
 \STATE {\bf Require}: Diffusion-model $  \{\Ncal(\rho(x_{t},t ;\theta),\sigma^2(t)) \}_{t=T+1}^1$, pre-trained model $\{\Ncal(\rho(x_{t},t;\theta_{\pre}),\sigma^2(t))\}_{t=T+1}^1$, batch size $m$, parameter $\alpha \in \RR^+$, learning rate $\gamma$
  \STATE Set a model $\{v_t(\cdot;\theta)\}$ to learn optimal soft value function, and a model $\{ p_t(\cdot|\cdot ;\theta)\} $ to learn optimal polices. 
   \STATE {\bf Initialize}: $\theta_1 = \theta_{\pre}$
      \FOR{$s=\{1,\cdots,S\}$}
      \STATE Collect $m$ samples $\{x^{(i)}_t(\theta) \}_{t=T+1}^0$ from a current diffusion model (i.e., generating by sequentially running polices $\{\Ncal(\rho(x_{t},t;\theta),\sigma^2(t))\}_{t=T+1}^1$ from $t=T$ to $t=0$) 
      \STATE Set $v_0 = r $ 
      {\small 
     \begin{align*}
    \phi_{s+1}  & \leftarrow \phi_s -\gamma \nabla_{\phi} \sum_{t=T+1}^1\sum_{i=1}^m \left \{\frac{v_t(x^{(i)}_t;\phi)}{\alpha} + \log p_t(x^{(i)}_{t-1}|x^{(i)}_t;\theta_s) -  \frac{v_{t-1}(x^{(i)}_{t-1};\phi_s)}{\alpha} -  \log p^{\pre}_t(x^{(i)}_{t-1}|x^{(i)}_t)  \right \}^2    
  |_{\phi_s} , \\ 
     \theta_{s+1}  & \leftarrow \theta_s -\gamma \nabla_{\theta} \sum_{t=T+1}^1\sum_{i=1}^m  \left \{\frac{v_t(x^{(i)}_t;\phi_s)}{\alpha} +  \log p_t(x^{(i)}_{t-1}|x^{(i)}_t;\theta) - \frac{v_{t-1}(x^{(i)}_{t-1};\phi_s)}{\alpha}  -  \log p^{\pre}_t(x^{(i)}_{t-1}|x^{(i)}_t)  \right \}^2    
|_{\theta_s}, 
     \end{align*}
     } 
 
   \ENDFOR 
  \STATE {\bf Output}: $\{p_t(x_{t-1}|x_t;\theta_S) \}_t$
\end{algorithmic}
\end{algorithm}

Now, we explain how to apply path consistency learning (PCL) \citep{nachum2017bridging} to fine-tune diffusion models. In the Gflownets literature \citep{bengio2023gflownet}, it seems that this variant is utilized as either a detailed balance or a trajectory balance loss, as discussed in \citet{mohammadpour2023maximum,tiapkin2023generative,deleu2024discrete}. However, to the best of our knowledge, the precise formulation of path consistency learning in the context of fine-tuning diffusion models has not been established. { Therefore, we start by elucidating the rationale of PCL. Subsequently, we provide a comprehensive explanation of the PCL. Finally, we discuss its connection with the literature on Gflownets. } 

\paragraph{Motivation.} Here, we present the fundamental principles of the PCL.
To start with, we prove the following lemma, which characterizes soft-value functions and soft-optimal policies recursively.  

\begin{lemma}[1-step Consistency Equation] 
\begin{align}\label{eq:detailed}
    \left(\frac{ v_{t}(x_{t})}{\alpha}\right) + \log p^{\star}_t(x_{t-1}|x_{t}) =  \left( \frac{ v_{t-1}(x_{t-1})}{\alpha}\right) + \log  p^{\pre}_t(x_{t-1}|x_{t})
\end{align}    
\end{lemma}
\begin{proof}
We consider the marginal distribution with respect to $x_t$ and $x_{t-1}$ induced by the soft-optimal policy, and denote it $l(x_t,x_{t-1}) \in \Delta(\Xcal \times \Xcal)$. Then, it is clearly $l(x_t)l(x_{t-1} \mid x_{t}) = l(x_{t-1} )l(x_{t} \mid x_{t-1}).$
Now, from Theorem~\ref{thm:key2} that characterizes marginal distributions, and Theorem~\ref{thm:key3} that characterizes posterior distributions, this results in :
\begin{align*}
     \underbrace{\frac{1}{C}\exp\left (\frac{ v_{t}(x_{t})}{\alpha} \right )p^{\pre}_t(x_t) }_{\text{Marginal distribution at t}}\times   \underbrace{p^{\star}_t(x_{t-1}|x_{t})}_{\text{Optimal policy}} = \underbrace{\frac{1}{C} \exp\left (\frac{ v_{t-1}(x_{t-1})}{\alpha} \right ) p^{\pre}_{t-1}(x_{t-1})}_{\text{Marginal distribution at t-1} }\times  \underbrace{p^{\pre}_{t-1}(x_{t}|x_{t-1})}_{\text{Posterior distribution}}
\end{align*}
Rearranging yields:
\begin{align*}
    \frac{1}{C}\exp\left (\frac{ v_{t}(x_{t})}{\alpha} \right )\times p^{\star}_t(x_{t-1}|x_{t}) = \frac{1}{C} \exp\left (\frac{ v_{t-1}(x_{t-1})}{\alpha} \right ) \times  p^{\pre}_t(x_{t-1}|x_{t})
\end{align*}
Taking the logarithm, the statement is concluded. 
\end{proof}

\paragraph{Algorithm.} Being motivated by the relation in \eqref{eq:detailed}, after initializing $v_{0}=r$, we obtain the recursive equation: 
{\small 
\begin{align}\label{eq:loss_gflownets}
    (v_t,p^{\star}_t) = \argmin_{g^{(1)}:\Xcal \to \RR,  g^{(2)}:\Xcal \to \Delta(\Xcal)}\EE_{x_t \sim u_t}\left[  \left\{  \frac{ g^{(1)}(x_{t})}{\alpha}  + \log g^{(2)}(x_{t-1}|x_{t})-  \frac{ v_{t-1}(x_{t-1})}{\alpha} -  \log  p^{\pre}_t(x_{t-1}|x_{t})\right\}^2 \right ]
\end{align}
} 
where $u_t \in \Delta(\mathcal{X})$ is any exploratory roll-in distribution. Based on this algorithm, we outline the entire algorithm in Algorithm~\ref{alg:detailed_balance_loss}.

We make several important remarks regarding ~\pref{alg:detailed_balance_loss}. Firstly, while we use on-policy data collection, technically, any policy can be used in this off-policy algorithm, like reward-weighted MLE. Secondly, in practice, it might be preferable to utilize a sub-trajectory from $x_{t}$ to $x_{t-k}$ based on the following expression:
\begin{align*}
        &  \log p^{\star}_{t-k+1}(x_{t-k}|x_{t-k+1}) + \cdots + \log p^{\star}_t(x_{t-1}|x_{t}) + \left(\frac{ v_{t}(x_{t})}{\alpha}\right)\\
        &=  \left( \frac{ v_{t-k}(x_{t-k})}{\alpha}\right) + \log  p^{\pre}_{t-k+1}(x_{t-k}|x_{t-k+1})+\cdots + \log  p^{\pre}_t(x_{t-1}|x_{t}), 
\end{align*}
which is an extension of \eqref{eq:detailed}. The loss function based on the above $k$-step consistency equation could make training faster without learning value functions at every time point, as noted in the literature in PCL. In the extreme case (i.e., when we recursively apply it with $t=T$), we obtain the following.

\begin{corollary}[$T$-step consistency]
\begin{align}\label{eq:t-step}
        &  \log p^{\star}_t(x_{0}|x_{1}) + \cdots + \log p^{\star}_T(x_{T-1}|x_{T}) + \log p^{\star}_T(x_{T})   \\
        &=  \left( \frac{r(x_{0})}{\alpha}\right) + \log  p^{\pre}_1(x_{0}|x_{1})+\cdots + \log  p^{\pre}_T(x_{T-1}|x_{T}) + \log  p^{\pre}_T(x_{T}).  \nonumber 
\end{align}
\end{corollary}
\begin{proof}
 We consider the marginal distribution with respect to $x_T,\cdots,x_0$ induced by the soft-optimal policy, and denote it $l(x_{T},\cdots,\cdots,x_0)\in \Xcal \times \cdots \times \Xcal$. We have 
 \begin{align}\label{eq:whole_equatoin}
    l(x_{T})l(x_{T-1}\mid x_{T})\cdots l(x_0|x_1) = l(x_0)l(x_1\mid x_0)\cdots l(x_T\mid x_{T-1})   
 \end{align}
From Theorem~\ref{thm:key2} that characterized marginal distributions, and Theorem~\ref{thm:key3} that characterize posterior distributions, the left hand side of \eqref{eq:whole_equatoin} is equal to 
 \begin{align*}
  \underbrace{p^{\star}_T(x_T) }_{\text{Marginal distribution at T}} \times   \underbrace{p^{\star}_{T-1}(x_{T}|x_{T-1})}_{\text{Optimal policy at $T-1$}} \times \cdots \times  \underbrace{ p^{\star}_1(x_{0}|x_{1})}_{\text{Optimal policy at $1$ }} 
 \end{align*}
and the right-hand side in \eqref{eq:whole_equatoin} is equal to 
 \begin{align*}
    \underbrace{\frac{\exp(r(x_0)/\alpha)}{C} p^{\pre}_0(x_0)}_{\text{Marginal distribution at 0}} \times  \underbrace{p^{\pre}_0(x_1 \mid x_0)}_{\text{Posterior distribution}} \times\cdots  \times \underbrace{p^{\pre}_{T-1}(x_T \mid x_{T-1})}_{\text{Posterior distribution}} . 
 \end{align*}
 By rearranging the term, we obtain \eqref{eq:t-step}. 
\end{proof}

\paragraph{Comparison with Gflownets.} In the Gflownets literature, similar losses are used. For instance, the loss derived from \eqref{eq:detailed} or \eqref{eq:t-step} is commonly known as a detailed balance loss \citep{bengio2023gflownet} or a trajectory loss \citep{malkin2022trajectory}, respectively. 

Note in general, the literature in Gflownets primarily focuses on sampling from unnormalized models  (distributions proportional to $\exp(r(x))$). Hence, reference policies (i.e, $\{p^{\pre}_t\}$ ) or latent states (i.e., $x_{T:1}$ before $x_0$) are introduced without relying on pre-trained diffusion models. In contrast, in our context, we use policies derived from pre-trained diffusion models as reference policies, leveraging them as our prior knowledge.

\section{Fine-Tuning Settings Taxonomy } \label{sec:settings}

So far, we implicitly assume we have access to reward functions. However, these functions are often unknown and need to be learned from data. We classify several settings in terms of whether reward functions are available or, if not, how they could be learned. This section is summarized in Figure~\ref{fig:practical}. 

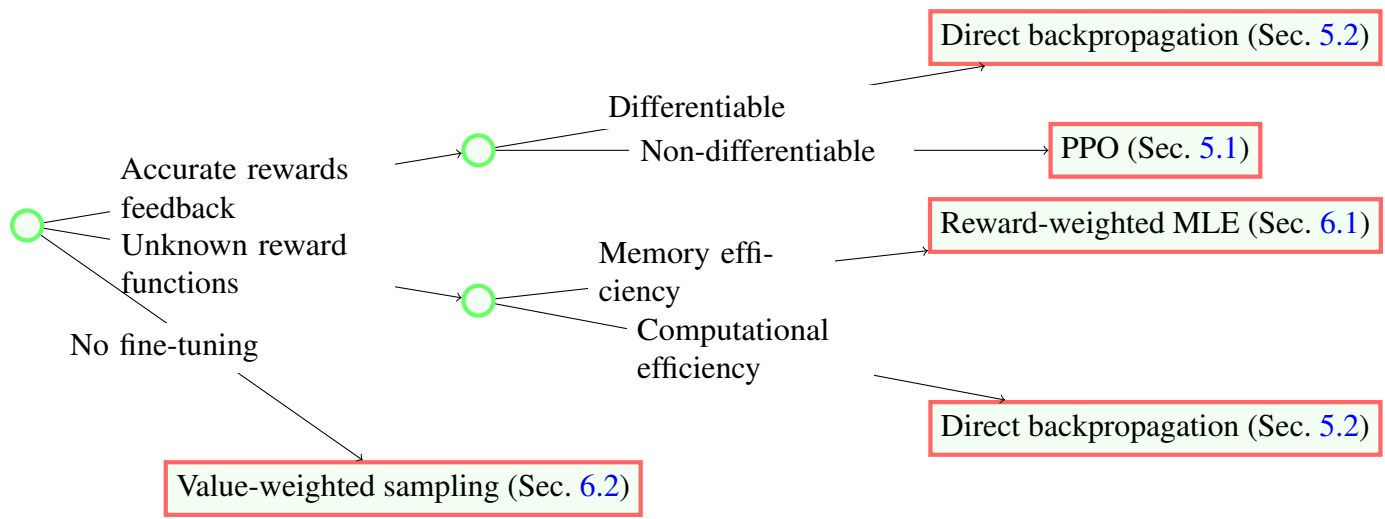
\begin{figure}[!th]
\tikzset{
}
\begin{tikzpicture}[
mycircle/.style={circle, draw=green!60, fill=green!5, ultra thick, minimum size=3mm,align=left},
mycircle2/.style={rectangle, draw=red!60, fill=green!5, ultra thick, minimum size=7mm,align=left},
myline/.style={dotted, blue!60,->},
  predicate/.style = {fill=white,  text width=3.5cm,  align=left }, 
  predicate2/.style = {fill=white,  text width=3cm,  align=left }
]
  \node[mycircle] (???) at (-1,0) {} ;
  \node[mycircle] (cat1) at (5,1) {} ;
  \node[mycircle] (cat6) at (5,-1) {} ;
  \node[mycircle2] (cat2) at (14,2.5) {Direct backpropagation (Sec.~\ref{sec:reward_prop})} ;
  \node[mycircle2] (cat3) at (14,1) {PPO (Sec.~\ref{sec:PPO})} ;
  \node[mycircle2] (cat4) at (14,0) {Reward-weighted MLE (Sec.~\ref{sec:reward-weighted})} ;
  \node[mycircle2] (cat7) at (14,-2.7) {Direct backpropagation (Sec.~\ref{sec:reward_prop})} ;
\node[mycircle2] (cat8) at (4,-3.5) {Value-weighted sampling (Sec.~\ref{sec:value})} ;
  \draw[->] (???) --node[predicate] {Accurate rewards feedback} (cat1) ;
  \draw[->] (???) --node[predicate] {Unknown reward functions } (cat6) ;
  \draw[->] (???) --node[predicate] {No fine-tuning} (cat8) ;
  \draw[->] (cat1) --node[predicate] {Differentiable } (cat2) ;
  \draw[->] (cat1) --node[predicate] {Non-differentiable} (cat3) ;
\draw[->] (cat6) --node[predicate2] {Memory efficiency} (cat4) ;
\draw[->] (cat6) --node[predicate2] {Computational efficiency} (cat7) ;
\end{tikzpicture}
\caption{ Practical recommendation of RL-based algorithms to optimize downstream reward functions.  } 
\label{fig:practical}
\end{figure}

\subsection{Fine-Tuning with Known, Differentiable Reward Functions}

When accurate differentiable reward functions are available (e.g., standard in computer vision tasks), our primary focus is typically on computational and memory efficiency. In general, we advocate for employing reward backpropagation (c.f.~\pref{alg:main}) for its computational efficiency. Alternatively, if memory efficiency is a significant concern, we suggest using PPO due to its stability.

\subsection{Fine-Tuning with Black-Box Reward Feedback}\label{sec:black_box}

Here, we explore scenarios where we have access to accurate but non-differentiable (black-box) reward feedback, often found in scientific simulations. The emphasis still remains primarily on computational and memory efficiency. In such scenarios, due to the non-differentiability of reward feedback, we recommend using PPO or reward-weighted MLE for the following reasons. 

In the preceding section, we advocated for the use of reward backpropagation due to its computational advantages. However, when dealing with non-differentiable feedback, the advantage of reward backpropagation may diminish. This is because learning from such feedback requires learning differentiable reward functions to make the algorithm work as we discuss in \pref{sec:reward_prop}. However, obtaining a differentiable reward function can be challenging in many domains. For example, in molecular property prediction, molecular fingerprints are informative features, and accurate reward functions can be derived by training simple neural networks on these features \citep{pattanaik2020molecular}. Yet, these mappings, grounded in prior scientific understanding, are non-differentiable, thereby restricting their applicability.

In contrast, PPO and reward-weighted MLE allow for policy updates without the explicit requirement to learn a differentiable reward function, offering a significant advantage over reward backpropagation.

\vspace{-3mm}
\subsection{Fine-Tuning with Unknown Rewards Functions}\label{sec:unknown}

When reward functions (or computational proxies in \pref{sec:black_box}) are unavailable, we must learn from data with reward feedback. In such cases, compared to the previous scenarios, two important considerations arise: 
\begin{itemize}
    \item  Not only computational or memory efficiency but also feedback efficiency (i.e., sample efficiency regarding reward feedback) is crucial. 
    \item  Given that learned reward feedback may not generalize well outside the training data distributions, it is essential to constrain fine-tuned models to prevent significant divergence from diffusion models. In these situations, not only soft PPO and direct backpropagation but also methods discussed in \pref{sec:sum_planning_conservative}, such as reward-weighted MLE, can be particularly effective. 
\end{itemize}
Further, the scenario involving unknown reward functions can be broadly categorized into two cases, which we will discuss in more detail.

\vspace{-2mm}
\begin{wrapfigure}{r}{0.43\textwidth}
\includegraphics[width=0.9\linewidth]{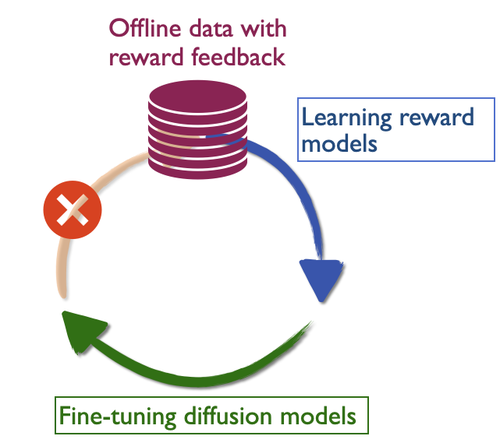}
\label{fig:offline}
\caption{Offline scenario}
\includegraphics[width=0.9\linewidth]{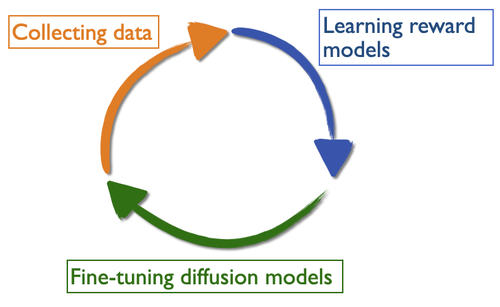}
\label{fig:online}
\caption{Online scenario}
\end{wrapfigure}

\paragraph{Offline scenario (Figure 4).} In many scenarios, even in cases where reward functions are unknown, we often have access to offline data with reward feedback $\{x^{(i)},r(x^{(i)})\}$. One straightforward approach is to perform regression with neural networks and build a learned regressor $\hat r$. However, this learned reward function $\hat r$ might not generalize well beyond the distribution of the offline data. In regions outside this distribution, $\hat r$ could assign a high value even when the actual reward $r$ is low due to high uncertainty. Consequently, we could be easily misled by out-of-distribution samples if we solely optimize $\hat r$.

One approach to alleviate this issue is to adopt a pessimistic (i.e., conservative) strategy, as proposed in \citet{uehara2024bridging}. This technique has been widely applied in the field of offline RL \citep{levine2018reinforcement}. The main idea is to penalize $\hat r$ outside the distributions using techniques such as adding an explicit penalty term \citep{yu2020mopo,chang2021mitigating}, bootstrapping, or more sophisticated methods \citep{kumar2020conservative,xie2021bellman,rigter2022rambo,uehara2021pessimistic}. By doing so, we can avoid being fooled by out-to-distribution regions while still benefiting from the extrapolation capabilities of reward models.

\paragraph{Online scenario (Figure 5).}

Consider situations where reward functions are unknown, but we can gather data online. These scenarios are often referred to as a lab-in-the-loop setting.

In such scenarios, \citet{uehara2024feedback} proposes an iterative procedure comprising three steps: (1) collecting feedback data by exploring new areas (i.e., obtaining $x$ from current diffusion models and corresponding $r(x)$), (2) learning the reward function from the collected feedback data, and (3) fine-tuning current diffusion models by maximizing the learned reward function enhanced by an optimistic bonus term, using algorithms discussed in \pref{sec:sum_planning}. In contrast to standard online PPO, where feedback data is directly used in the fine-tuning process, actual reward feedback is only used in step (2) and not in step (3). An additional key aspect is the use of an optimistic reward function that encourages exploration beyond the distributions of the current diffusion model. This deliberate separation between reward learning and fine-tuning steps, coupled with the use of optimism, significantly enhances feedback efficiency.

\section{Connection with Classifier Guidance}\label{sec:connection}

Classifier guidance \citep{dhariwal2021diffusion} is commonly used for conditional generation in diffusion models. Interestingly, RL-based methods share a close connection with classifier guidance. More specifically, in this section, following \citet{zhao2024adding}, we elucidate how RL-based methods can be applied to conditional generation. Furthermore, from this unified viewpoint, we highlight that classifier guidance is seen as a value-weighted sampling in \pref{sec:sum_planning}.

\subsection{Classfier Guidance}\label{sec:classfier-based}

In this section, we first introduce classifier guidance. 
Classifier guidance utilizes an unconditional pre-trained model for conditional generation. More specifically, we consider a scenario where we have a pre-trained diffusion model, which enables us to sample from $p^{\pre}(x)$, and data with feedback $\{x,y \}$, where $y$ represents the new label we want to condition on. Our objective here is to sample from  $p(\cdot|y) \in [\Ycal \to \Delta(\RR^d)]$, 
\begin{align}\label{eq:original}
    p(\cdot \mid y) \propto p_{y}(y|\cdot)p^{\pre}(\cdot). 
\end{align}
where $p_y: \Xcal \to \Delta(\Ycal)$ denotes the conditional distribution of $y$ given $x$. 

\paragraph{Motivation.}  In classifier guidance, we consider the following policy: 
\begin{align*}
    p^{\text{gui}}_t(\cdot |x_t,y) &=\Ncal\left ( \rho(x_t, t ;\theta_{\pre})+ \sigma^2(t)  \nabla _x \log q_y(x_t, t ), \sigma^2(t) \right),\\
    q_y(x, t) &:=  \EE_{\{ p^{\pre}_t \} , Y \sim p_y(\cdot|x_0) }[\mathrm{I}(Y=y)|x_t=x]. 
\end{align*} 
 Under the continuous-time formulation of diffusion models, applying Doob's h-transform \citep{rogers2000diffusions}, it is shown that we can sample from $p(x|y)$ by sequentially running $\{p^{\text{gui}}_t(\cdot |x_t,y)\}_{t=T+1}^1$.

\paragraph{Algorithm.} Specifically, the algorithm comprises two steps. It's important to note that this is not a fine-tuning method but rather an inference-time technique as we freeze the pre-trained model.
\begin{enumerate}
    \item Learning $q_y(x,t)$ through regression by gathering data containing $\{x_t,y\}$. Here, for each $(x_0,y)$, using a forward policy for the pre-training (from $0$ to $T$), we obtain $x_t$ starting from $x_0$. 
    \item During inference, sequentially execute $\{p^{\text{gui}}_t(\cdot |x_t,y)\}_{t=T+1}^1$ by leverage the fact that compared to the policy in pre-trained diffusion models, the mean of  $\{p^{\text{gui}}_t(\cdot |x_t,y)\}_{t=T+1}^1$ is just shifted with an additional term $\nabla_x \log q_y(x,t)$.
\end{enumerate}

Notably, the method described above can be reduced to value-weighted sampling (cf. Section~\ref{sec:value}), when $r(x,y)=\log p(y|x)$. Inspired by such an observation, we apply an RL-based fine-tuning method for conditional generation as below.

\subsection{RL-Based Fine-Tuning for Conditional Generation}

In this section, following~\citet{zhao2024adding}, we explain an RL-based fine-tuning approach for conditional generation. The core idea is that setting $r(x,y)= \log p(y|x)$ allows us to sample from the target conditional distribution \eqref{eq:original}, drawing on the insights discussed in Section~\ref{sec:sum_planning}.

More precisely, introducing a distribution over $y$ as $q(\cdot)$, we formulate the following RL problem:
\begin{align}\label{eq:key_plnanning2}
       \{p^{\star}_t\}_t=   & \argmax_{\{p_t \in  [(\Ycal,\Xcal) \to \Delta(\Xcal)] \}_{t=T+1}^1 }\EE_{\{p_t(\cdot|\cdot,y) \},y\sim \Pi}[\log p(y|x_0)   - \Sigma_{t=T+1}^1 \mathrm{KL}(p_t(\cdot|x_t,y )\|p^{\pre}_t (\cdot |x_t))  ]. 
\end{align}
Compared to the objective function \eqref{eq:key_plnanning} in \pref{sec:sum_planning}, the main differences are: (1) policy space is expanded to accommodate additional control (i.e., $y$): $p_t(\cdot|\cdot,\cdot)\subset [\RR^d \times \Ycal \to \Delta(\RR^d)]$, (2) the fine-tuning objectives are set as log-likelihoods: $\log p(y|x)$. Then, following \pref{thm:main}, we can establish the following result.

\begin{corollary}
The distribution induced by $\{p^{\star}_t\}_{t=T+1}^1$ (i.e., $\int \{\prod_{t=T+1}^1 p^{\star}_t(x_{t-1}\mid x_t,y) \}dx_{1:T}$ is the target conditional distribution~\eqref{eq:original}. 
\end{corollary}

The above corollary suggests that for conditional generation, we can utilize various off-the-shelf algorithms discussed in Section~\ref{sec:sum_planning} and \ref{sec:sum_planning_conservative}, such as PPO, reward backpropagation, and reward-weighted MLE, as well as value-weighted sampling.

\section{Connection with Flow-Based Diffusion Models}\label{sec:flow_matching}

We elucidate the close relationship between bridge matching and fine-tuning, as summarized in Table~\ref{tab:summary_connection}.  Bridge (flow) matching \citep{liu2022flow,tong2023conditional,lipman2022flow,liu2022let,shi2024diffusion,albergo2022building}, has recently gained popularity as a distinct type of diffusion model from score-matching-based ones. These works suggest that training and inference speed will be accelerated due to the ability to define a more direct trajectory from noise to data distribution than score-matching-based diffusion models.

The connection between bridge matching and fine-tuning becomes evident in the continuous-time formulation. To begin, we first provide a brief overview of the continuous-time formulation of score-matching-based diffusion models.

\begin{table}[!t]
    \centering
       \caption{
      Comparison between training score/flow-based diffusion models and fine-tuning. We optimize the parameter $\theta$ in $\PP^{\theta}$, which represents the induced distribution over trajectories associated with $\theta$. 
       }
    \label{tab:summary_connection}
    \begin{tabular}   
    {p{0.15\textwidth}p{0.18\textwidth}p{0.5\textwidth}}
         &   Loss & Note of $\QQ$  \\ \midrule 
   Score-based training    &  $\KL(\QQ^{\mathrm{time}}\|\PP^{\theta})$  & $\QQ^{\mathrm{time}}$: Time-reversal of reference SDE  (reference SDE is defined from $0$ to $T$ where the initial distribute at $0$ is data distribution) \\ 
   Flow-based training & $\KL(\QQ^{\mathrm{doob}}\|\PP^{\theta})$ & $\QQ^{\mathrm{doob}}$: Coupling between the reference SDE from $T$ to $0$ conditioning at $0$ (dist. over $x_{T:1}|x_0$) and a data distribution at time $0$ (dis. over $x_0$). Note that reference SDE is defined from $T$ to $0$.   \\
   Fine-tuning  & $\KL(\PP^{\theta} \| \QQ^{\mathrm{fine}})$ & $\QQ^{\mathrm{fine}}$: Coupling between the pre-trained SDE from $T$ to $0$ conditioned on $0$ (dist. over $x_{T:1}|x_0$) and a data distribution at time $0$ (dis. over $x_0$)  \\
       \bottomrule
    \end{tabular}
\end{table}

\subsection{Continuous Time Formulation of Score-Based diffusion models} 

We provide a brief overview of the continuous-time formulation of diffusion models \citep{song2021maximum}, further expanding \pref{sec:training}. In this formulation, we initially define the forward reference SDE, spanning from time $0$ to $T$ (e.g., variance exploding (VE) process or variance preserving (VP) process), by setting an initial distribution at $0$ as the data distribution. Additionally, the reference SDE is tailored to converge to an easy distribution at time $T$, such as the normal distribution. Subsequently, we consider the time-reversal SDE \citep{anderson1982reverse}, progressing from $T$ to $0$. If we could learn this time-reversal SDE, it would enable us to sample from the data distribution by starting with the easy initial distribution (at $T$) and following the time-reversal SDE from $T$ to $0$ at inference time. 

\paragraph{Training.}
Now, the remaining question is how to learn this time-reversal SDE. Here, by representing the induced distribution over trajectories from $T$ to $0$ with this time-reversal SDE as $\QQ^{\text{time}}$, our goal is to train a new SDE, parameterized by $\theta$ in the neural network, such that the induced distribution $\PP^{\theta}$ over trajectories associated with $\theta$ aligns with $\QQ^{\text{time}}$. Specifically, the objective is to minimize the following losses:
\begin{align}\label{eq:kl_loss}
\argmin_{\theta}\KL(\QQ^{\text{time}} \| \PP^{\theta}).
\end{align}
With some algebra, leveraging the observation that the time-reversal SDE comprises the original forward SDE and the scoring term of the \emph{marginal} distribution \citep{anderson1982reverse} as in \eqref{eq:time_reversal},  \citet{song2021maximum} demonstrates that this KL loss \eqref{eq:kl_loss} is approximated as 
\begin{align}\label{eq:loss_diffusion2}
  \hat \theta_{\pre} =  \argmin_{\theta} \EE_{t\sim \mathrm{Uni}([0,T]),x_t \sim q_{t\mid 0}(x_t|x_0), x_0 \sim p^{\pre} }[ \{\sigma^{\diamond}_t\}^2 \| \nabla_{x_t} \log q_{t\mid 0}(x_t \mid x_0) - s_{\theta}(x_t,t) \| ^2     ], 
\end{align}
which is equal to \eqref{eq:loss_diffusion} where the weight $\lambda(t)$ is $\{\sigma^{\diamond}_t\}^2$ (for notation, refer to \pref{sec:training}). 

\subsection{Flow-Based (Bridge-Based) Diffusion Models}

{ We adopt the explanation outlined in \citet{liu2022flow}. In this framework, we begin by considering a reference SDE (from $T$ to $0$), $x_t=z_{T-t}$  where $z_t$ is defined as Brownian motion:
\begin{align}\label{eq:gaussian}
   d z_t =   dw_t, \quad t\in [0,T].
\end{align}
Subsequently, we consider the coupling of the reference SDE conditioned on state at time $0$ (i.e., distribution of $x_{0:T}$ conditional on $x_0$) and the data distribution at time $0$, denoting the resulting distribution as $\QQ^{\text{doob}}$. Informally, this $\QQ^{\text{doob}}$  is characterized as follows. 

\paragraph{Informal characterization of $\QQ^{\text{doob}}$.}
When discretizing the time step, by denoting the distribution induced by the reference SDE as $p^{\mathrm{ref}}(x_T,\cdots,x_0)$, we define 
\begin{align*}
\QQ^{\text{doob},T}(x_T,\cdots,x_0):= p^{\mathrm{ref}}(x_T,\cdots , x_1 \mid x_0)p^{\pre}(x_0).
\end{align*}
The distribution $\QQ^{\text{doob}}$ above is a limiting distribution obtained by making the discretization step small.

\paragraph{Training.} In bridge-based diffusion models, our objective is to train a new SDE, parametrized by $\theta$ (e.g., a neural network), to align the induced distribution $\PP^{\theta}$ with $\QQ^{\text{doob}}$. Specifically, bridge-based training aims to minimize
\begin{align}\label{eq:kl_loss2}
\argmin_{\theta} \KL(\QQ^{\text{doob}} \| \PP^{\theta}). 
\end{align}

While the method to minimize the above is currently unclear, we can derive a trainable objective through algebraic manipulation. Initially, we observe that the reference SDE \eqref{eq:gaussian}, conditioned on state $x_0$ at time $0$ ($z_T = b$), is expressed as:
\begin{align}\label{eq:conditional_SDE}
      t\in [0,T];  dz^b_t =  g^b(z^b_t,t) dt + dw_t, \quad g^b(z^b_t,t)=\frac{b-z^b_t  }{T-t}, 
\end{align}
using the seminal ``Doob's h-transform'' \citep{rogers2000diffusions}. This SDE is commonly known as the Brownian bridge. Now, we consider a new SDE parameterized by $\theta$: 
\begin{align}\label{eq:conditional_SDE2}
    t\in [0,T]; dz_t = s(z_t,t;\theta)dt + dw_t.
\end{align}
Then, we can demonstrate that minimizing the KL loss in \eqref{eq:kl_loss2} is equivalent to 
 \begin{align}\label{eq:loss_flow}
 \argmin_{\theta}\EE_{t \sim \mathrm{Uni}[0,T], z_{t}\sim q^{x_0}_t, x_0 \sim p^{\pre}  }\left[ \|g^{x_0}(z_t,t)-s(z_t,t;\theta)  \|^2_2 \right ]
\end{align}
where $q^{x_0}_t$ denotes the induced distribution at time $t$ following SDE \eqref{eq:conditional_SDE} conditioned on $x_0$, and $s: \Xcal \times [0,T] \to \Xcal$ is the model we aim to optimize.}

\paragraph{Flow-based diffusion models.}  The aforementioned formulation has been originally proposed as a bridge-based SDE \citep{liu2022let} because the conditional SDE is often referred to as a Brownian bridge. Flow-based diffusion models proposed in \citet{lipman2022flow, tong2023conditional} share a close relationship with bridge-based diffusion models. These flow-based models use a target SDE \eqref{eq:conditional_SDE} and an SDE with a parameter \eqref{eq:conditional_SDE2} without stochastic terms (i.e., no $dw_t$) and aim to minimize the loss function \eqref{eq:loss_flow}. 

\paragraph{Comparison with score-based diffusion models.} Both score-based diffusion models and bridge-based diffusion models aim to minimize KL divergences. Despite its similarity, in contrast to score-based diffusion models, bridge-based diffusion models target $\QQ^{\text{doob}}$ rather than the time-reversal SDE $\QQ^{\text{time}}$. The distribution $\QQ^{\text{doob}}$ is often considered preferable because it can impose an SDE that navigates efficiently from $0$ to $T$ (e.g.,  a Brownian bridge with minimal noise). This can expedite the learning process since we lack direct control over the time-reversal SDE in score-based diffusion models.

\subsection{Connection with RL-Based Fine-Tuning} 

RL-based fine-tuning shares a close relationship with flow-based training. To illustrate this, consider the reference SDE (from $T$ to $0$) 
following the pre-trained diffusion model. Then, akin to flow-based training, we consider the coupling of the reference SDE conditioned on state at time $0$ (i.e., distribution over $x_{0:T}$ conditioned on $x_0$) and the \emph{target} distribution at time $0$ (i.e., $p_r(\cdot) \propto \exp(r(\cdot)/\alpha)p^{\pre}(\cdot)$), denoting the induced distribution as $\QQ^{\text{fine}}$.  

\paragraph{Informal characterization of $\QQ^{\text{fine}}$. }
Informally, $\QQ^{\text{fine}}$ is introduced as follows. When discretizing the time step, by denoting the distribution induced by the reference SDE as $p^{\mathrm{ref}}(x_T,\cdots,x_0)$, we define 
\begin{align*}
\QQ^{\text{fine},T}(x_T,\cdots,x_0):= p^{\mathrm{ref}}(x_T,\cdots , x_1 \mid x_0)p_r(x_0).
\end{align*}
The distribution $\QQ^{\text{fine}}$ above is a limiting distribution obtained by making the discretization step small.

\paragraph{Training.} Now, we introduce a new SDE parameterized by $\theta$ such as an neural network to align the induced distribution $\PP^{\theta}$ with $\QQ^{\text{fine}}$. Actually, in the continuous formulation, the RL-problem \eqref{eq:key_plnanning} in \pref{sec:sum_planning} is equal to solving 
\begin{align*}
\argmin_{\theta} \KL( \PP^{\theta}\| \QQ^{\text{fine}}).
\end{align*}
This formalization has been presented in \citet{uehara2024finetuning}. Intuitively, as we see in \pref{thm:key3}, this is expected because we see that the posterior distribution induced by pre-trained models: 
\begin{align*}
    p^{\mathrm{ref}}(x_T,\cdots , x_1 \mid x_0)\,(=\,p^{\mathrm{ref}}_{T-1}(x_T|x_{T-1})p^{\mathrm{ref}}_{T-2}(x_{T-1}\mid x_{T-2})\cdots p^{\mathrm{ref}}_0(x_1 \mid x_0  )), 
\end{align*}
remains preserved after fine-tuning based on the RL-formulation \eqref{eq:key_plnanning}.

\paragraph{Comparison with flow-based training.}
In contrast to flow-based training, RL-based fine-tuning minimizes the inverse KL divergence rather than the KL divergence. Intuitively, this is because, unlike $\QQ^{\text{time}}$ or $\QQ^{\text{doob}}$ , we cannot sample from $\QQ^{\text{fine}}$ (recall that a marginal distribution at time $0$ in $\QQ^{\text{fine}}$, i.e., $p_r$, is unnormalized); on the other hand,  the marginal distributions at time $0$ in both $\QQ^{\text{time}}$ and $\QQ^{\text{doob}}$ are data distributions.  

\begin{remark}[Extension to $f$-divergence]
Each of these methods can be technically extended by incorporating more general divergences beyond KL divergence. For instance, in the context of fine-tuning, please see \citet{tang2024fine}.
\end{remark}

\section{Connection with Sampling from Unnormalized Distributions}\label{sec:sampling}

 Numerous studies delve into sampling from an unnormalized distribution proportional to $\exp(r(x))$, commonly known as the Gibbs distribution, extensively discussed in the computational statistics and statistical physics community \citep{robert2014statistical}. This issue is pertinent to our work, as during fine-tuning, the target distribution is also formulated in the Gibbs distribution form in \pref{eq:original_target2}, which is proportional to $\exp(r(x))p^{\pre}(x)$. 
 
The literature employs two main approaches to tackle this challenge: Markov Chain Monte Carlo (MCMC) and RL-based methods. In particular, the RL-based diffusion model fine-tuning we have discussed so far is closely related to the latter approach. In this section, we explain and compare these two approaches. 

\subsection{Markov Chain Monte Carlo (MCMC)}  In MCMC, a Markov chain is constructed to approximate a target distribution such that the equilibrium distribution of the Markov chain converges to the target distribution. However, dealing with high-dimensional domain spaces is challenging for MCMC. A common strategy involves leveraging gradient information, such as the Metropolis-adjusted Langevin algorithm (MALA) \citep{besag1994comments} or Hamiltonian Monte Carlo (HMC) \citep{neal2011mcmc,girolami2011riemann}. 
MALA shares similarities with classifier guidance and value-weighted sampling, as both methods depend on the first-order information from reward (or value) functions.

\paragraph{Can we use MCMC for fine-tuning diffusion models?} 
In the context of fine-tuning diffusion models, while it may seem intuitive to sample from the target distribution \eqref{eq:original_target2} (i.e., $p_r \propto \exp(r(x))p^{\pre}(x)$) using MCMC, it's not straightforward for the following reasons:
\begin{itemize}
    \item We lack of an analytical form of $p^{\pre}(\cdot)$. 
    \item Additionally, even if we can estimate $p^{\pre}(\cdot)$ in an unbiased manner, as in \citet{chen2018neural}, the mixing time for obtaining a single sample in MCMC might be lengthy.
\end{itemize}
Therefore, we explore creating a generative model (i.e., diffusion model) for this task using RL-based fine-tuning so that we can easily sample from $p_{r}(\cdot)$ during inference (i.e., simulating the time-reversal SDE \eqref{eq:time_reversal} with policies from $t=T+1$ to $t=1$), without relying on MCMC. {This addresses the concern of MCMC by avoiding the need to estimate $p^{\pre}(\cdot)$ and minimizing inference time, albeit at the expense of increased training time.
}

\subsection{RL-Based Approaches} While MCMC has historically been widely used for this purpose, recent works have proposed an RL-based approach or its variant (e.g., \citet{bernton2019schr,heng2020controlled,huang2023reverse,vargas2023denoising}). For example, a seminal study \citep{zhang2021path} introduced a method very similar to reward backpropagation. In their work, they initially define Brownian motion as a reference SDE (from $T $ to $0$), similar to the pre-trained diffusion models in our context. Then, by appropriately setting rewards, \citet{zhang2021path} aims to learn a new SDE (from $T$ to $0$) such that we can sample from the target distribution at the endpoint $0$. 

These approaches offer advantages over MCMC due to their ability to minimize inference time (amortization) by constructing a generative model comprising policies in the training time and solely executing learned policies during inference. Unlike MCMC, these methods significantly reduce the inference cost as we don't need to be concerned about the mixing time. Additionally, RL-based approaches have the potential to effectively handle high-dimensional, complex (i.e., multi-modal) distributions by harnessing the considerable expressive power arising from latent states.

Despite these advantages, it is often unclear how to select a reference SDE in these works. In contrast, in RL-based fine-tuning of diffusion models, we directly leverage pre-trained diffusion models as the reference SDEs, i.e., as our prior knowledge.

\paragraph{Key message.} Finally, it is worth noting that we can technically utilize any off-the-shelf RL algorithms for training. For example, \citet{zhang2021path} employs reward backpropagation, while works in Gflownets typically leverage PCL, as we mention in \pref{sec:gflownet}. However, additionally, it is possible to use PPO or reward-weighted MLE mentioned in \pref{sec:sum_planning}.

\section{Closely Related Directions}

{ Lastly, we highlight several closely related directions we have not yet discussed. }  

\paragraph{Aligning text-to-image models.} We recognize that current efforts in fine-tuning diffusion models primarily revolve around aligning text-to-image models using human feedback \citep{lee2023aligning,wallace2023diffusion,wu2023better,yang2023diffusion}. In this paper, our goal is to consider a more general and fundamental formulation that is not tailored to any particular task.

\paragraph{Diffusion models for RL.} 

{
Diffusion models are well-suited for specific reinforcement learning applications due to their ability to model complex and multimodal distributions as polices
\citep{janner2022planning,ajay2023is,wang2022diffusion,hansen2023idql,du2024learning,zhu2023diffusion}.
For example, \citet{wang2022diffusion} use conditional diffusion models as policies and demonstrate their effectiveness in typical offline RL benchmarks. While RL-based fine-tuning also aims to maximize certain reward functions, the primary focus of this tutorial covers methods that leverage pre-trained diffusion models for this purpose.
}

{ \paragraph{RL from human feedback (RLHF) for language models.} 

RLHF is widely discussed in the context of language models, where many algorithms share similarities with those used in diffusion models~\citep{ouyang2022training,bai2022constitutional,casper2023open}. For instance, both domains commonly employ PPO as a standard approach, and reward-weighted training (or reward-weighted MLE) often serves as a basic baseline approach. Despite these similarities, nuanced differences exist mainly because diffusion models typically operate in a continuous input space rather than a discrete one. Moreover, techniques like value-weighted sampling emerge uniquely within the context of diffusion models.
}

\section{Summary}

In this article, we comprehensively explain how fine-tuning diffusion models to maximize downstream reward functions can be formalized as a reinforcement learning (RL) problem in entropy-regularized Markov decision processes (MDPs). Based on this viewpoint, we elaborate on the application of various RL algorithms, such as PPO and reward-weighted MLE, specifically tailored for fine-tuning diffusion models. Additionally, we categorize different scenarios based on how reward feedback is obtained. While this categorization is not always explicitly mentioned in many existing works, it is crucial when selecting appropriate algorithms. Finally, we discuss the relationships with several related topics, including classifier guidance, Gflownets, path integral control theory, and sampling from unnormalized distributions.

\bibliographystyle{chicago}
\bibliography{rl}

\appendix

\end{document}